\documentclass{article}

\usepackage{iclr2027_conference}
\iclrfinalcopy
\usepackage{times}
\usepackage{latexsym}
\usepackage[T1]{fontenc}
\usepackage[utf8]{inputenc}
\usepackage{microtype}
\usepackage{subcaption}
\usepackage{inconsolata}
\usepackage{graphicx}
\graphicspath{{figures/}}
\usepackage{amsmath,amssymb,amsthm}
\usepackage{booktabs}
\usepackage{multirow}
\usepackage{array}
\usepackage{xcolor}
\usepackage{url}
\usepackage{enumitem}
\usepackage{placeins}
\usepackage{float}
\usepackage{caption}
\usepackage{subcaption}

\usepackage{tikz}
\usepackage{pgfplots}
\pgfplotsset{compat=1.16}
\usepackage{algorithm}
\usepackage{algorithmic}
\usepackage[most]{tcolorbox}
\usepackage{pifont}
\usepackage{colortbl}
\usepackage{hyperref}
\hypersetup{colorlinks=false, citebordercolor={0 1 0}, linkbordercolor={1 0 0}, urlbordercolor={0 1 1}, pdfborder={0 0 1}}
\renewcommand{\cite}{\citep}
\numberwithin{equation}{section}

\newtheorem{theorem}{Theorem}
\newtheorem{lemma}[theorem]{Lemma}
\newtheorem{corollary}[theorem]{Corollary}
\newtheorem{proposition}[theorem]{Proposition}

\newtheorem{assumption}{Assumption}

\newcommand{\GEAF}{\widehat{\mathrm{GEAF}}}

\title{Understanding Rollout Error in Graph World Models}

\author{\textbf{Xinyuan Song}$^{1}$ \quad
    \textbf{Zekun Cai}$^{2}$ \\
    $^{1}$Emory University, Atlanta, GA, USA \quad
    $^{2}$The University of Tokyo, Tokyo, Japan \\
    \texttt{xinyuan.song@emory.edu, caizekun@csis.u-tokyo.ac.jp} \\
}

\begin{document}
\maketitle

\begin{abstract}
World models are increasingly used for planning by predicting future states under candidate actions. Many planning environments, however, are naturally graph-structured, with agents, tools, skills, routes, and dependencies connected through fixed or evolving relations. Understanding how prediction errors accumulate in such graph-structured rollouts is therefore critical for reliable long-horizon planning.In this work, we systematically study how rollout errors propagate in Graph World Models (GWMs). We analyze both fixed- and dynamic-edge regimes and derive topology-aware error bounds that characterize the roles of graph structure, learned dynamics, and evolving edges. We then validate these theoretical findings across synthetic graph topologies and heterogeneous agent-graph testbeds, showing that rollout error and planning regret increase with horizon and that dynamic-edge training is important when graph structure evolves. Together, these results characterize when GWMs remain reliable under long-horizon planning and when graph structure amplifies rollout errors. Motivated by these findings, we propose \textbf{Error-Aware GWM}, which combines spectral regularization, rollout consistency, and critical-node weighting, and improves long-horizon stability without sacrificing one-step accuracy. Experiments show that Error-Aware GWM consistently improves long-horizon rollout stability without sacrificing prediction accuracy.
\end{abstract}


\section{Introduction}

\label{sec:intro}

World models have become an important approach to planning by predicting future states under candidate actions~\cite{ha2018worldmodels,hafner2018dreamer,hafner2020dreamer,hafner2023dreamerv3,chua2018pets,schrittwieser2020muzero,wu2024ivideogpt,bruce2024genie}. Many planning environments, however, are naturally graph-structured: knowledge-graph memories support agent planning~\cite{anokhin2024arigraph}, external APIs form tool-use spaces~\cite{li2023apibank,qin2024toolllm,liu2024agentbench,trivedi2024appworld}, skill graphs organize reusable capabilities~\cite{huang2025asg,cascade2025,rl_sisl_2025,xia2026graspgraphstructuredskillcompositions}, and communication networks connect interacting entities. In these environments, the world state includes not only node states but also the relations through which actions and information propagate. Reliable planning therefore depends critically on controlling rollout error, since prediction errors can accumulate over time and affect downstream decisions~\cite{asadi2018lipschitz,janner2019trust,kakade2002approx}.

Graph World Models (GWMs) provide a natural framework for such planning by autoregressively predicting graph states $G_t=(V,E_t,X_t,A_t)$~\cite{feng2025gwm}. Given a learned transition model $\mathcal{F}_\theta(\hat{G}_k,a_k)$, a GWM predicts $\hat{G}_1,\ldots,\hat{G}_H$ to evaluate candidate action sequences. Yet how rollout error propagates in GWMs remains largely unexplored. Prior work has focused mainly on static graph tasks~\cite{liu2023one,chen2024llaga}, graph generation and dynamic graph representation learning~\cite{you2018graphrnn,vignac2022digress,rossi2020temporal,roland2022}, and domain-specific graph planning~\cite{zhang2020worldgraph,kgwm2021}. Feng et al.~\cite{feng2025gwm} introduced GWMs with action nodes, but did not study long-horizon error propagation under coupled node-edge prediction. This leaves an important gap for both long-horizon error propagation and dynamic-edge (DE) settings, where predicted edges become part of the state used for subsequent predictions.

Graph structure makes compounding error harder to analyze than in vector-state world models. Standard model-based RL analyses typically summarize rollout error with a scalar Lipschitz constant over a fixed-dimensional state~\cite{asadi2018lipschitz}, but graph rollouts introduce three additional challenges. First, error propagation depends explicitly on topology: the same local perturbation may remain localized on a chain but spread widely in highly connected graphs, consistent with prior studies of graph propagation, dynamics, stability, and rewiring~\cite{battaglia2016interaction,battaglia2018relational,sanchezgonzalez2018graph,oono2020graph,chamberlain2021grand,curvaturerewiring2024,piorf2025}. For a message-passing GWM, this dependence separates into a topology term $\rho(A)$ and a model term $\prod_\ell\|W_\ell\|_2$. Second, when edges are predicted, node and edge errors become coupled: node-feature error $e^X_k=\|\hat X_k-X_k\|_F$ affects future edge prediction, while edge error $e^A_k=\|\hat A_k-A_k\|_F$ changes subsequent message passing. Third, these rollout errors ultimately matter because they affect downstream planning decisions over horizon $H$. These observations motivate our graph-specific analysis: we define $\GEAF=\rho(A)\prod_\ell\|W_\ell\|_2$ for fixed-edge error amplification, introduce a joint operator $B$ for the recursion of $(e^X_k,e^A_k)$ in dynamic-edge GWMs, and connect both forms of rollout error to planning regret.

We analyze accumulative rollout error in GWMs theoretically and empirically. Our graph error amplification factor $\GEAF=\rho(A)\prod_\ell\|W_\ell\|_2$ separates topology-dependent amplification from the model-dependent spectral norm product, while a joint node-edge operator $B$ captures node-to-node, edge-to-node, node-to-edge, and edge-to-edge error propagation. This distinguishes fixed-edge (FE) models, where $M_X=0$ and error propagation is node-only, from dynamic-edge (DE) models, where $M_X>0$ activates node-edge cross-coupling. We evaluate these findings on seven synthetic graph topologies using GCN~\cite{kipf2017semi}, MPNN~\cite{gilmer2017neural}, GPS~\cite{rampasek2022recipe}, and ActionNode GWM~\cite{feng2025gwm}, with both FE and DE simulators and heterogeneous agent-graph testbeds inspired by tool/API and agent benchmarks~\cite{li2023apibank,qin2024toolllm,liu2024agentbench,trivedi2024appworld}. The results show that long rollouts amplify prediction errors into larger planning errors, while evolving graphs require DE training and exhibit the node-edge coupling predicted by $B$. Overall, GWM stability depends jointly on topology, horizon, and rollout regime rather than architecture alone.

We further propose \textbf{Error-Aware GWM} to improve rollout stability on high-spectral-radius graphs. The method controls topology-induced error amplification through spectral regularization, with rollout consistency and critical-node weighting further preserving predictive accuracy. Empirically, it prevents the divergence observed in vanilla GCN rollouts while maintaining low long-horizon error. 

Our contributions are summarized as follows:
\begin{itemize}[left = 0em]

\item \textbf{Systematic study of rollout error in GWMs (Section~\ref{sec:gwm}).}
We study how errors accumulate in Graph World Models under both fixed-edge and dynamic-edge rollouts, establishing a unified framework for graph-structured planning.

\item \textbf{Topology-aware error theory and analysis (Section~\ref{sec:theory}).}
We derive graph-valued rollout-error bounds that characterize topology-dependent amplification, node-edge coupling, and their effects on long-horizon planning.

\item \textbf{Error-Aware GWM (Section~\ref{sec:methods}).}
Motivated by our theoretical and empirical findings, we propose Error-Aware GWM to control error amplification and improve long-horizon rollout stability.

\item \textbf{Extensive evaluation (Section~\ref{sec:experimental_setup}).}
We validate the theory and method across diverse graph topologies, FE/DE rollout regimes, and heterogeneous agent-graph testbeds.

\item \textbf{Real-world benchmark evaluation (Section~\ref{sec:RQ5_realworld}).}
We further evaluate GWMs on real-world graph benchmarks to examine their applicability beyond controlled environments.

\end{itemize}


\section{Graph World Model Framework}

\label{sec:gwm}

\subsection{Preliminaries: Agent Graph World}

A \textbf{Graph World Model} (GWM) represents an agent environment as a graph and learns its state transition for planning. At time $t$, the graph state is
\begin{equation}
G_t=(V,E_t,X_t,A_t).
\label{eq:gwm_state}
\end{equation}
Here, $V=\{v_1,\ldots,v_N\}$ is a fixed set of $N$ nodes representing agents, tools, skills, or tasks; $E_t\subseteq V\times V$ is the edge set; $X_t\in\mathbb{R}^{N\times d}$ contains the $d$-dimensional node features; and $A_t\in\{0,1\}^{N\times N}$ is the adjacency matrix associated with $E_t$. Node features describe quantities such as agent status, tool outputs, skill availability, task progress, and failure signals. Edges describe calling, dependency, routing, or communication relations.

At each step, the planner selects an action $a_t\in\mathcal{A}$, where $\mathcal{A}$ denotes the action space. The GWM is a learned transition function
\begin{equation}
\mathcal{F}_\theta:\mathcal{G}\times\mathcal{A}\rightarrow\mathcal{G},\qquad \hat G_{t+1}=\mathcal{F}_\theta(G_t,a_t),
\label{eq:gwm_transition}
\end{equation}
where $\mathcal{G}$ is the graph-state space and $\theta$ denotes the model parameters. Given an initial state $\hat G_0=G_t$ and a candidate action sequence $(a_0,\ldots,a_{H-1})$, an $H$-step rollout is defined recursively by
\begin{equation}
\hat G_{k+1}=\mathcal{F}_\theta(\hat G_k,a_k),\qquad k=0,\ldots,H-1.
\label{eq:gwm_rollout}
\end{equation}
The resulting trajectory $(\hat G_1,\ldots,\hat G_H)$ is used to evaluate the candidate action sequence.

\subsection{GWM Rollout Mechanism}

\label{sec:instantiations}

Following Feng et al.~\cite{feng2025gwm}, each action $a_t$ is represented by an action node whose feature vector encodes the selected action and whose connections specify its targets in the current graph. A \textbf{node-level action} changes the state of one or more nodes, an \textbf{edge-level action} changes one or more relations, and a \textbf{graph-level action} changes a graph-wide property. We consider two settings according to whether the adjacency matrix is fixed or predicted during rollout.

\paragraph{Fixed-Edge GWM (FE-GWM).}
In FE-GWM, the edge set and adjacency matrix remain fixed during rollout, so $E_{k+1}=E_k=E$ and $A_{k+1}=A_k=A$, while the model predicts only node features. Let $D$ be the degree matrix of $A$ and $\tilde A=D^{-1/2}AD^{-1/2}$. A representative two-layer GCN transition is
\begin{equation}
\hat X_{k+1}=\tanh(\tilde A\hat X_kW_1+\mathbf{1}(Ua_k)^\top)W_2+\xi_k,
\label{eq:fe_transition}
\end{equation}
where $\hat X_k\in\mathbb{R}^{N\times d}$ is the predicted node-feature matrix, $a_k\in\mathbb{R}^{d_a}$ is the action embedding, $U\in\mathbb{R}^{d_h\times d_a}$ maps the action to the hidden space, $W_1$ and $W_2$ are learnable weights, and $\xi_k\in\mathbb{R}^{N\times d}$ denotes one-step prediction error or rollout noise. If an action targets only part of the graph, its representation is injected only into the connected nodes. 

\paragraph{Dynamic-Edge GWM (DE-GWM).}
In DE-GWM, both node features and graph structure are predicted:
\begin{equation}
\hat X_{k+1}=F_\theta^X(\hat X_k,\hat A_k,a_k),\qquad \hat A_{k+1}=F_\theta^A(\hat X_k,\hat A_k,a_k),
\label{eq:de_transition}
\end{equation}
where $F_\theta^X$ and $F_\theta^A$ denote the node- and edge-transition functions, respectively. The node transition uses the current predicted adjacency $\hat A_k$ for message passing. For edge prediction, let $h_{k,i}\in\mathbb{R}^{d_h}$ denote the hidden representation of node $v_i$ at rollout step $k$. A representative edge predictor is
\begin{equation}
\hat A_{k+1,ij}=\sigma(h_{k,i}^{\top}Qh_{k,j}),
\label{eq:edge_prediction}
\end{equation}
where $Q\in\mathbb{R}^{d_h\times d_h}$ is learnable and $\sigma$ is the sigmoid function, so $\hat A_{k+1,ij}$ is the predicted strength or probability of edge $(v_i,v_j)$. Given ground-truth next-step states $(X_{k+1},A_{k+1})$, the model can be trained with
\begin{equation}
\mathcal{L}=\|F_\theta^X(X_k,A_k,a_k)-X_{k+1}\|_F^2+\lambda_e\|F_\theta^A(X_k,A_k,a_k)-A_{k+1}\|_F^2+\lambda_w\|\theta\|_2^2,
\label{eq:de_training}
\end{equation}
where $\lambda_e\ge0$ controls the edge-prediction loss and $\lambda_w\ge0$ controls weight regularization.


\section{Error Amplification Theory}

\label{sec:theory}

Our analysis follows the assumptions stated in Section~\ref{sec:assumptions}. Figure~\ref{fig:intuition} summarizes the resulting framework, from topology-dependent error amplification and autoregressive rollout to the joint node-edge error operator $B$ and planning regret, while distinguishing the FE and DE regimes.

\begin{figure*}[!ht]
  \centering
  \includegraphics[width=\textwidth]{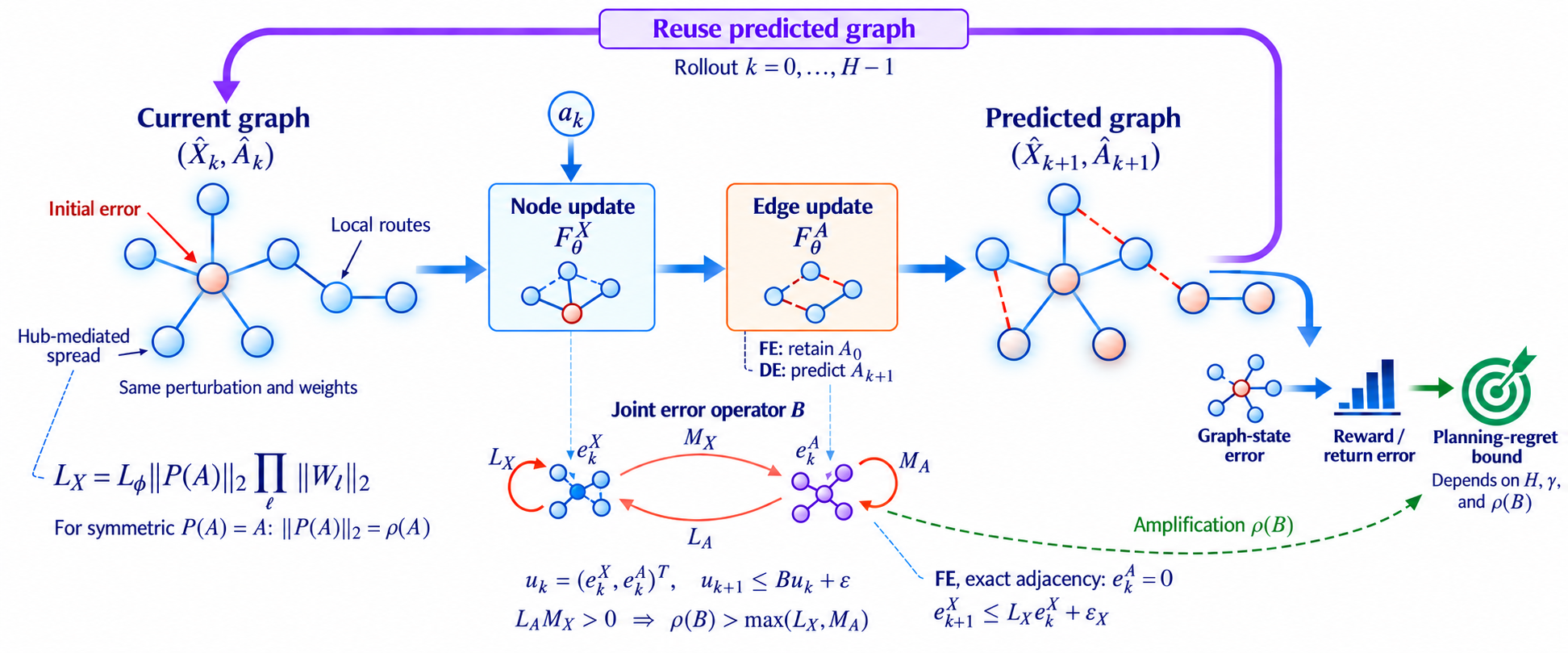}
  
  \caption{\textbf{Topology-dependent rollout error in GWMs.}
  Autoregressive rollouts propagate node and edge errors through graph structure. The joint operator $B$ describes node-edge error propagation and its effect on planning regret; FE models reduce to node-only propagation, whereas DE models introduce coupled node-edge errors.}
  \label{fig:intuition}
  
\end{figure*}

\subsection{Problem Formulation}

\label{sec:problem_formulation}

We study the accumulation of prediction error over an $H$-step GWM rollout defined in Equations~\eqref{eq:gwm_state}--\eqref{eq:gwm_rollout}. Let $(X_k,A_k)$ be the true graph state at step $k$ and $(\hat X_k,\hat A_k)$ its prediction. We define the node and edge rollout errors as
\begin{equation}
u_k=(e^X_k,e^A_k)^\top, \quad e^X_k=\|\hat X_k-X_k\|_F,\quad e^A_k=\|\hat A_k-A_k\|_F,
\label{eq:node_edge_errors}
\end{equation}
Our goal is to characterize how $u_k$ evolves with rollout horizon and graph structure under the FE and DE settings defined in Section~\ref{sec:instantiations}. Specifically, we seek bounds of the form $u_{k+1}\preceq Bu_k+\epsilon$, where $B$ describes graph-dependent error propagation and $\epsilon$ captures one-step prediction error. We then study when this recursion remains stable and how the resulting rollout error affects planning regret.

\subsection{Fixed-Edge Rollout Error}

\label{sec:t1_fixed_edge}

We first analyze the fixed-edge (FE) regime, where topology is held fixed and only node features are predicted. Rollout error then propagates through a fixed message-passing operator. The bound separates the topology-dependent part of amplification from the model-dependent part.

\begin{lemma}[Fixed-edge node error bound]
\label{lem:t1_fixed_edge}
Under the FE-GWM defined in Section~\ref{sec:instantiations} and Assumptions~\ref{assump:onestep_error}--\ref{assump:message_passing}, let $A_t\equiv A$ and let the activation have Lipschitz constant $L_\sigma$. Then, for any rollout step $k\ge0$,
\begin{equation}
e^X_k\le L_X^k e^X_0+\epsilon_X\frac{L_X^k-1}{L_X-1},\quad L_X=L_\sigma\underbrace{\rho(A)}_{\text{topology}}\underbrace{\prod_\ell\|W_\ell\|_2}_{\text{model}}.
\label{eq:fixed_edge_node_bound}
\end{equation}
\end{lemma}

Lemma~\ref{lem:t1_fixed_edge} separates FE-GWM error amplification into a topology factor $\rho(A)$ and a model factor $\prod_\ell\|W_\ell\|_2$. The former captures error propagation through graph structure, while the latter captures amplification by the learned transition. We therefore define the \textbf{Graph Error Amplification Factor (GEAF)} as
\begin{equation}
\GEAF(G)=\rho(A)\prod_\ell\|W_\ell\|_2.
\label{eq:geaf_def}
\end{equation}
For fixed model weights, $\GEAF$ shows how graph topology alone can change long-horizon error growth. The proof of Lemma~\ref{lem:t1_fixed_edge} is given in Appendix~\ref{app:proof_t1_fixed_edge}.

\subsection{Dynamic-Edge Rollout Error}

\label{sec:t2_dynamic_edge}

We next analyze the dynamic-edge (DE) regime, where the GWM predicts both node states and graph structure. Here, rollout errors are coupled: node-feature errors can affect future edge prediction, and edge errors can change future message passing. We capture this interaction with a joint node-edge error operator.

\begin{lemma}[Joint node-edge error bound]
\label{lem:t2_joint_operator}
Under the DE-GWM defined in Section~\ref{sec:instantiations} and Assumptions~\ref{assump:onestep_error} and~\ref{assump:edge_update}, let $u_k=(e^X_k,e^A_k)^\top$. Then the joint rollout error satisfies
\begin{equation}
u_{k+1}\preceq Bu_k+\epsilon,\qquad B=\begin{pmatrix}L_X & L_A\\ M_X & M_A\end{pmatrix},\qquad \epsilon=(\epsilon_X,\epsilon_A)^\top.
\label{eq:t2_recursion}
\end{equation}
The spectral radius of the joint error operator $B$ is
\begin{equation}
\rho(B)=\frac{1}{2}\left[(L_X+M_A)+\sqrt{(L_X-M_A)^2+4L_AM_X}\right].
\label{eq:joint_operator_spectral_radius}
\end{equation}
If $L_AM_X>0$, then $\rho(B)>\max(L_X,M_A)$.
\end{lemma}

so node-edge coupling increases the error amplification rate beyond either uncoupled component. Lemma~\ref{lem:t2_joint_operator} shows that DE-GWM rollout error is governed by a coupled operator rather than a scalar node-error recursion. The off-diagonal terms $L_A$ and $M_X$ encode edge-to-node and node-to-edge error propagation. In FE models, there is no edge-prediction head, so $M_X=0$ and the recursion reduces to node-only prediction. In DE models, $M_X>0$ activates node-edge feedback. The proof is provided in Appendix~\ref{app:proof_t2_joint_operator}.

\subsection{GEAF as a Proxy for the Joint Operator}

\label{sec:t3_geaf_proxy}

The joint operator $B$ characterizes node-edge error amplification, but its entries may not be directly available. We therefore study $\GEAF$ as a topology-aware proxy for its dominant amplification scale.

\begin{lemma}[GEAF proxy bound]
\label{lem:t3_geaf_proxy}
Under Assumptions~\ref{assump:compact_state}--\ref{assump:edge_update}, suppose $\|A\|_2\ge R_X$, where $R_X$ bounds the non-topological contribution to the first-row error term. Then
\begin{equation}
\rho(B)\le \GEAF(G)\left(1+\frac{R_X}{\|A\|_2}\right)\le 2\GEAF(G).
\label{eq:geaf_proxy_bound}
\end{equation}
When the model weights are fixed across graph topologies, the ordering induced by $\GEAF(G)$ follows that of $\rho(A)$.
\end{lemma}

Lemma~\ref{lem:t3_geaf_proxy} bounds the joint amplification rate $\rho(B)$ by $\GEAF$ up to a constant factor in this regime. For FE-GWM, the edge-error terms vanish and $\rho(B)=L_X=\GEAF$ by construction. For DE-GWM, node-edge coupling breaks this identity, making the relation between $\rho(B)$ and $\GEAF$ nontrivial. The proof is given in Appendix~\ref{app:proof_t3_geaf_proxy}.

\subsection{Planning Regret Under Graph Rollout Error}

\label{sec:t4_regret}

We next connect the joint rollout error in Equation~\eqref{eq:t2_recursion} to planning performance under Assumptions~\ref{assump:reward}--\ref{assump:graph_distance}. Recall that $u_k=(e_k^X,e_k^A)^\top$ is defined in Equation~\eqref{eq:node_edge_errors} and $B$ is the joint error operator defined in Equation~\eqref{eq:t2_recursion}.

\begin{lemma}[Planning regret bound]
\label{lem:t4_regret}
Under Assumptions~\ref{assump:reward}--\ref{assump:graph_distance}, let $c=(\alpha,\beta)^\top$, $\tilde\epsilon=\|(\epsilon_X,\epsilon_A)^\top\|_2$, and let $\kappa(B)$ satisfy $\|B^j\|_2\le\kappa(B)\rho(B)^j$. Let $\hat\pi_\theta$ and $\pi^*$ be the policies selected under the learned and true dynamics, respectively. Then
\begin{equation}
J^*(\pi^*)-J^*(\hat\pi_\theta)\le 2L_R\|c\|_2\kappa(B)\tilde\epsilon\,\Phi_H(\gamma,\rho(B))+2\epsilon_R\frac{1-\gamma^H}{1-\gamma},
\label{eq:planning_regret_bound}
\end{equation}
where $L_R,\epsilon_R,\gamma,H$ are defined in Assumptions~\ref{assump:reward}--\ref{assump:discount}, $\rho(B)$ is defined in Equation~\eqref{eq:joint_operator_spectral_radius}, and
\begin{equation}
\Phi_H(\gamma,\rho)=\frac{1}{\rho-1}\left[\frac{1-(\gamma\rho)^H}{1-\gamma\rho}-\frac{1-\gamma^H}{1-\gamma}\right].
\label{eq:phi_h}
\end{equation}
\end{lemma}

If $\gamma\rho(B)>1$, the rollout-error term grows super-linearly with $H$. Lemma~\ref{lem:t4_regret} shows that small one-step errors can produce large planning regret when the joint error operator is expansive. The spectral radius $\rho(B)$ controls error growth across rollout steps, while $\gamma$ controls how strongly future errors affect the planning objective. The proof is given in Appendix~\ref{app:proof_t4_regret}.

\begin{corollary}[Error level versus growth rate]
\label{cor:slope_vs_level}
Under the FE recursion in Lemma~\ref{lem:t1_fixed_edge}, if the rollout error is not dominated by the one-step error floor, then
\begin{equation}
\frac{\partial \log e^X_k}{\partial k}\to \log L_X.
\label{eq:slope_vs_level}
\end{equation}
\end{corollary}

Thus, the log-error growth rate can reveal topology-dependent amplification even when the final error approaches the floor $\epsilon_X/(1-L_X)$. We therefore report GrowthSlope as an exploratory diagnostic. The proof is given in Appendix~\ref{app:proof_slope_vs_level}.

\section{Dataset and Evaluation Protocol}

\label{sec:dataset}

\subsection{Graph Topologies}

\label{sec:data_topo}

We evaluate GWMs on seven canonical graph families with distinct structural and spectral properties: chain, tree, grid, small-world, scale-free, star, and complete graphs. Each graph contains $N=50$ nodes, with three outer seeds per topology. Together, these families span adjacency spectral radii from $\rho(A)=2.0$ to $49.0$ and cover sparse, clustered, hub-dominated, and dense connectivity regimes. Table~\ref{tab:topo_stats} summarizes their topology statistics and rollout behavior.

\begin{table*}[!ht]
\centering
\caption{Topology statistics for the seven synthetic graph families at $N=50$. $\rho(A)$ is the adjacency spectral radius, $\widehat{\mathrm{GEAF}}$ is the mean estimated GEAF over successful baseline runs, NodeMSE@32 is the median 32-step node prediction error, GrowthSlope is the median log-error slope from horizons 4 to 32, and Diverged counts rollout divergence across three seeds.}

\label{tab:topo_stats}
\small
\setlength{\tabcolsep}{5pt}
  \resizebox{0.8\textwidth}{!}{
\begin{tabular}{lrrrrrr}
\toprule
Topology & $N$ & $\rho(A)$ & $\widehat{\mathrm{GEAF}}$ & NodeMSE@32 & GrowthSlope & Diverged \\
\midrule
Chain       & 50 & $2.00$        & $8.7$   & $5.0 \times 10^{-4}$ & $-0.0075$ & 0/3 \\
Tree        & 50 & $2.53$        & $11.8$  & $4.0 \times 10^{-4}$ & $-0.0042$ & 0/3 \\
Grid        & 50 & $3.65$        & $16.0$  & $2.7 \times 10^{-4}$ & $-0.0071$ & 1/3 \\
Small-World & 50 & $4.16\pm0.05$ & $19.2$  & $4.5 \times 10^{-4}$ & $+0.0008$ & 1/3 \\
Scale-Free  & 50 & $5.94\pm0.15$ & $24.3$  & $2.9 \times 10^{-4}$ & $+0.0005$ & 3/3 \\
Star        & 50 & $7.00$        & $29.4$  & $1.9 \times 10^{-4}$ & $-0.0039$ & 0/3 \\
Complete    & 50 & $49.00$       & $224.9$ & $1.7 \times 10^{-4}$ & $0.0000$  & 0/3 \\
\bottomrule
\end{tabular}}

\end{table*}

Scale-free graphs follow preferential attachment~\cite{barabasi1999emergence}, while small-world graphs follow the Watts--Strogatz model~\cite{watts1998collective}; the remaining families are deterministic for fixed $N$. Each graph is initialized with node features $X_0\in\mathbb{R}^{N\times8}$. Full generation parameters are provided in Appendix~\ref{app:dataset_details}.

\subsection{Evaluation Environments}

\label{sec:data_sim}

We evaluate both rollout regimes studied in Section~\ref{sec:theory}. In the \textbf{FE simulator}, $A$ remains fixed and node states evolve according to
\begin{equation}
X_t=f_{\rm true}(X_{t-1},A)+\varepsilon_t,\qquad \varepsilon_t\sim\mathcal{N}(0,\sigma^2I),
\label{eq:fe_simulator}
\end{equation}
with $\sigma=0.01$, rollout length $T=50$, and evaluation horizons $H\in\{1,2,4,8,16,32\}$. In the \textbf{DE simulator}, both node states and edges evolve, activating the coupled error dynamics in Equation~\eqref{eq:t2_recursion}. The edge predictor follows Equation~\eqref{eq:edge_prediction}, and rollouts use $T=32$.

Beyond canonical graphs, we evaluate two heterogeneous agent-graph environments. The \textbf{agent calling-tree testbed}, motivated by tool/API-based agent systems~\cite{li2023apibank,qin2024toolllm,liu2024agentbench,trivedi2024appworld}, represents multi-stage task execution as a typed directed acyclic graph, with sink-node success rate $\mathrm{sr}_{\rm sink}$ as the primary metric. The \textbf{platform skill-graph testbed}, motivated by graph-structured skill systems~\cite{huang2025asg,cascade2025,xia2026graspgraphstructuredskillcompositions}, models dependencies among skills, tools, validators, artifacts, and recovery components, with skill success rate $\mathrm{sr}_{\rm skill}$ as the primary metric. Detailed graph construction, node and edge types, simulator parameters, and calibration results are provided in Appendix~\ref{app:dataset_details}.

\section{Error-Aware GWM}

\label{sec:methods}

Motivated by the error amplification analysis in Section~\ref{sec:theory}, we propose \textbf{Error-Aware GWM}, a training framework for improving long-horizon GWM stability. It applies to the FE and DE models defined in Section~\ref{sec:instantiations} without changing their basic transition architecture. The key idea is to retain the standard one-step prediction objective while explicitly controlling the sources of rollout error identified by our theory.

Specifically, Error-Aware GWM optimizes
\begin{equation}
\mathcal{L}_{\rm EA}
=
\mathcal{L}_{\rm pred}
+\lambda_{\rm spec}R_{\rm spec}
+\lambda_{\rm roll}R_{\rm roll}
+\lambda_{\rm crit}R_{\rm crit},
\label{eq:error_aware_objective}
\end{equation}
where $\mathcal{L}_{\rm pred}$ is the standard one-step prediction loss. The three additional terms target complementary sources of long-horizon error:
\begin{equation}
\small
R_{\rm spec}=\sum_{\ell}\|W_\ell\|_2^2,\quad
R_{\rm roll}=\sum_{h\in\mathcal{H}}\|\hat X_{t+h}^{(h)}-X_{t+h}\|_F^2,\quad
R_{\rm crit}=\sum_{v\in V}w_v\|\hat X_{t+1,v}-X_{t+1,v}\|_2^2.
\label{eq:error_aware_terms}
\end{equation}
Here, $\mathcal{H}$ denotes the training rollout horizons, $\hat X_{t+h}^{(h)}$ the $h$-step autoregressive prediction, and $w_v$ the normalized structural importance of node $v$. The three terms respectively control the model factor $\prod_\ell\|W_\ell\|_2$ in $\GEAF$, penalize accumulated rollout drift, and emphasize errors on structurally important nodes. These designs follow spectral control for Lipschitz stability~\cite{miyato2018spectral} and multi-step training under predicted trajectories~\cite{bengio2015scheduled,janner2019trust}.

For DE-GWM, we further include edge prediction and regularize node-edge interactions associated with the joint operator $B$ in Equation~\eqref{eq:t2_recursion}. Thus, Error-Aware GWM addresses both topology-dependent amplification and DE node-edge coupling. Full details are provided in Appendix~\ref{app:error_aware_details}.

\section{Experimental Setup}

\label{sec:experimental_setup}

\paragraph{Baselines.}
We compare six world-model baselines spanning non-graph dynamics, message passing, graph transformers, and action-node GWMs. The MLP world model (MLP-WM) ignores graph structure. The vanilla GCN world model is based on graph convolution~\cite{kipf2017semi}. The MPNN world model is based on neural message passing~\cite{gilmer2017neural}. We also include a GPS-style graph transformer~\cite{rampasek2022recipe}. The ActionNode GWM is based on the fixed-edge architecture of Feng et al.~\cite{feng2025gwm}. Our Error-Aware GWM adds spectral control, rollout consistency, and critical-node weighting to the same rollout interface.

\paragraph{Topologies.}
We evaluate all baselines on seven synthetic graph families: chain, tree, grid, small-world, scale-free, star, and complete graphs. Each graph has $N=50$ nodes. Scale-free graphs follow the preferential-attachment model~\cite{barabasi1999emergence}, and small-world graphs follow the Watts--Strogatz model~\cite{watts1998collective}. These topologies span a broad spectral range, from $\rho(A)=2.0$ for chain graphs to $\rho(A)=49.0$ for complete graphs, as summarized in Table~\ref{tab:topo_stats}.

\paragraph{Fixed-edge training.}
In the fixed-edge setting, the graph topology is fixed during rollout and each model predicts node states only. We train all six baselines on the FE simulator described in Section~\ref{sec:data_sim}. Evaluation uses rollout horizons $H\in\{1,2,4,8,16,32\}$, allowing us to measure both short-horizon prediction quality and long-horizon error accumulation, following the standard concern of compounding error in model-based rollout~\cite{asadi2018lipschitz,janner2019trust}.

\paragraph{Dynamic-edge training.}
In the dynamic-edge setting, each baseline is extended with the bilinear edge-prediction head in Equation~\eqref{eq:de_simulator_edge_head}. The model predicts both node features and edges, and is trained with the node loss, edge BCE loss, and the graph regularizer in Equation~\eqref{eq:de_graph_regularizer}. This setting is related to dynamic graph representation learning and temporal graph networks~\cite{rossi2020temporal,roland2022}, but differs in that the predicted graph is fed back as an autoregressive world-model state for rollout prediction. All six baselines are retrained on dynamic-edge rollouts across the six DE topologies and three seeds, yielding $6\times6\times3=108$ runs.

\paragraph{Heterogeneous agent testbeds.}
For the agent calling-tree and platform skill-graph testbeds, we evaluate six R-GCN variants: small, medium, deep, wide, action-node, and critical-node variants. These models handle typed nodes and typed edges in heterogeneous graphs~\cite{schlichtkrull2018modeling}, following the general message-passing view of graph networks~\cite{battaglia2018relational}. The testbeds are motivated by API/tool-use and interactive-agent benchmarks~\cite{li2023apibank,qin2024toolllm,liu2024agentbench,trivedi2024appworld}, but evaluate a different question: whether a graph world model can roll out the evolving execution graph. Each model is trained for 50 epochs with T-batched forward passes.

\section{RQ1: Does Topology Predict Rollout-Error Growth?}

\label{sec:rq1_topology_growth}

We first test whether topology predicts the \emph{rate} of rollout-error growth. Our theory suggests that $\rho(A)$ and $\GEAF$ should be reflected more directly in temporal growth than in terminal error, which can be affected by optimization and error floors. We therefore examine the relation between $\GEAF$ and GrowthSlope and distinguish this effect from standard GNN over-smoothing.

\subsection{GEAF Tracks Error-Growth Slope}
\label{sec:h1}

Figure~\ref{fig:h1} relates $\log\GEAF$ to GrowthSlope$_{4\to32}$, the slope of log-NodeMSE from horizon $4$ to $32$. Larger $\log\GEAF$ is associated with faster error growth, consistent with Corollary~\ref{cor:slope_vs_level}. Thus, $\GEAF$ does not necessarily determine terminal error, but captures the topology-dependent rate of long-horizon error amplification.

\begin{figure*}[!ht]
  \centering
  \resizebox{0.9\textwidth}{!}{
  \begin{minipage}{\textwidth}
    \centering
    \begin{subfigure}[t]{0.49\textwidth}
      \centering
      \includegraphics[width=\linewidth]{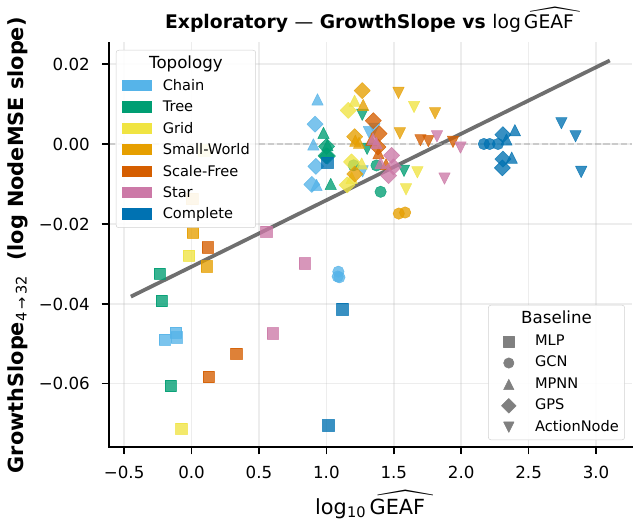}
      \caption{Rollout-error growth.}
      \label{fig:h1}
    \end{subfigure}
    \hfill
    \begin{subfigure}[t]{0.49\textwidth}
      \centering
      \includegraphics[width=\linewidth]{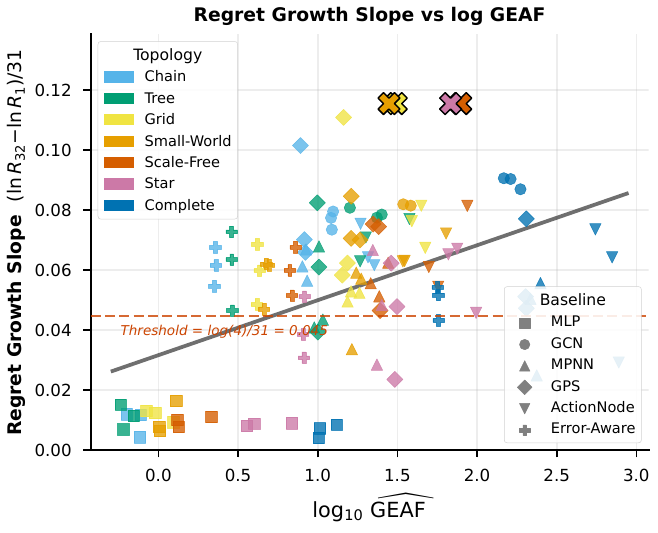}
      \caption{Planning-regret growth.}
      \label{fig:h4}
    \end{subfigure}
  \end{minipage}
  }
  
  \caption{\textbf{GEAF predicts long-horizon error and planning-regret growth.}
  (a) $\log\GEAF$ is positively associated with GrowthSlope$_{4\to32}$ of NodeMSE.
  (b) Higher $\log_{10}\GEAF$ is also associated with faster planning-regret growth from horizon $4$ to $32$.}
  \label{fig:geaf_growth}
  
\end{figure*}

\subsection{Topology effects strengthen with rollout horizon and edge density}
\label{app:exp17_rollout_length}

We study how rollout error changes with horizon. The experiment evaluates $H\in\{1,2,4,8,16,32,48\}$ across six baselines, seven topologies, and three seeds. Table~\ref{tab:rollout_length_scaling} reports NodeMSE@$H$ as a function of rollout horizon.

Low-risk topologies such as chain and tree remain stable over long horizons. High-risk topologies such as scale-free, star, and grid show sharp error growth at larger horizons. This complements the regret analysis: long-horizon behavior is strongly topology-dependent, and short-horizon performance does not fully characterize rollout stability.

\begin{table*}[!ht]

\centering
\caption{
\textbf{Rollout-length scaling across methods.}
Median NodeMSE over seven graph topologies and three random seeds at each rollout horizon. Lower is better.
}
\label{tab:rollout_length_scaling}

\resizebox{0.9\textwidth}{!}{
\begin{tabular}{lccccccc}
\toprule
\hline
Method
& $H=1$
& $H=2$
& $H=4$
& $H=8$
& $H=16$
& $H=32$
& $H=48$ \\
\midrule
MLP
& $1.299\times10^{-1}$
& $2.077\times10^{-2}$
& $1.062\times10^{-3}$
& $2.694\times10^{-4}$
& $2.598\times10^{-4}$
& $2.991\times10^{-4}$
& $2.959\times10^{-4}$ \\

GCN
& $5.815\times10^{-1}$
& $5.103\times10^{-1}$
& $5.294\times10^{-1}$
& $7.491\times10^{-1}$
& $9.948\times10^{-1}$
& $9.948\times10^{-1}$
& $9.945\times10^{-1}$ \\

MPNN
& $1.815\times10^{-3}$
& $5.883\times10^{-4}$
& $\mathbf{2.118\times10^{-4}}$
& $\mathbf{1.985\times10^{-4}}$
& $\mathbf{1.920\times10^{-4}}$
& $\mathbf{2.152\times10^{-4}}$
& $\mathbf{2.652\times10^{-4}}$ \\

GPS
& $3.111\times10^{-3}$
& $7.845\times10^{-4}$
& $3.608\times10^{-4}$
& $2.788\times10^{-4}$
& $3.158\times10^{-4}$
& $3.118\times10^{-4}$
& $3.731\times10^{-4}$ \\

ActionNode GWM
& $2.286\times10^{-3}$
& $6.851\times10^{-4}$
& $2.819\times10^{-4}$
& $2.588\times10^{-4}$
& $2.649\times10^{-4}$
& $2.976\times10^{-4}$
& $4.022\times10^{-4}$ \\

Error-Aware GWM
& $\mathbf{1.665\times10^{-3}}$
& $\mathbf{5.063\times10^{-4}}$
& $2.209\times10^{-4}$
& $2.030\times10^{-4}$
& $2.206\times10^{-4}$
& $2.495\times10^{-4}$
& $3.291\times10^{-4}$ \\
\hline
\bottomrule
\end{tabular}
}

\end{table*}

\subsection{Rollout Error Affects Planning Regret?}

\label{sec:rq2_planning_regret}
We next test whether prediction error affects downstream planning. Because GWMs select actions through imagined graph rollouts, accumulated prediction errors become harmful when they distort predicted returns. Figure~\ref{fig:h4} shows that runs with larger $\GEAF$ tend to exhibit faster regret growth from horizon $4$ to $32$, consistent with Lemma~\ref{lem:t4_regret}: stronger error amplification makes imagined trajectories less reliable as the horizon increases. Planning regret at horizon $32$ is also substantially larger than one-step regret across all topologies, confirming that the observed planning failure is primarily a long-horizon rollout effect rather than a one-step artifact.

\section{RQ2: When Does Dynamic-Edge Coupling Matter?}

\label{sec:RQ3_dynamic_edge}

We study when node-edge coupling matters. In FE models, edges are fixed and coupling is absent, whereas in DE models, predicted edges affect future node states and node states affect future edge predictions. We test whether DE training improves DE rollout accuracy and whether trained models exhibit the coupling predicted by Lemma~\ref{lem:t2_joint_operator}.

\subsection{DE Training Improves Rollout Accuracy and Activates Node-Edge Coupling}
\label{sec:h6}

We compare FE-trained and DE-trained models on DE rollouts using $n=36$ paired cells: ActionNode GWM and Error-Aware GWM across six topologies and three seeds. Table~\ref{tab:stream_a} shows that DE training improves NodeMSE@32 for both architectures, with every paired comparison favoring DE training. The combined improvement is $11.0\times$, supporting Lemma~\ref{lem:t2_joint_operator} that DE rollouts require learning both node dynamics and node-edge interactions.

We further test whether trained DE models exhibit the predicted coupling. Because the edge predictor gives $M_X>0$ and message passing gives $L_A>0$, Lemma~\ref{lem:t2_joint_operator} predicts
\begin{equation}
\rho(B)>\max(L_X,M_A)
\label{eq:de_spectral_witness}
\end{equation}
when $L_AM_X>0$. This condition holds in all 108 DE-trained cells, with $p<10^{-30}$ and mean excess $\approx155$ (Table~\ref{tab:stream_a}). The large margin shows that the learned models have strong node-edge coupling rather than behaving like nearly fixed-edge models.

\begin{table*}[!ht]
\centering
\caption{
Dynamic-edge training and node-edge coupling. The FE-trained / DE-trained NodeMSE ratio measures the improvement from DE training, with larger values indicating greater improvement.
}

\label{tab:stream_a}
\small
\resizebox{0.8\textwidth}{!}{
\begin{tabular}{lrcc}
\toprule
\hline
Test / Architecture & $n$ & FE-trained / DE-trained NodeMSE & Wilcoxon $p$ \\
\midrule
\multicolumn{4}{l}{\textit{DE-trained vs.\ FE-trained on DE rollouts}} \\
ActionNode GWM & 18 & $9.85\times$ & $3.8 \times 10^{-6}$ \\
Error-Aware GWM & 18 & $12.29\times$ & $3.8 \times 10^{-6}$ \\
\textbf{Combined} & \textbf{36} & $\mathbf{11.0\times}$ [9.4, 12.8] & $\mathbf{1.5 \times 10^{-11}}$ \\
\midrule
\multicolumn{4}{l}{\textit{Node-edge coupling in DE-trained models}} \\
All DE-trained models & 108 & $\rho(B)/\max(L_X,M_A)>1$ in 108/108; mean excess $\approx 155$ & $<10^{-30}$ \\
\hline
\bottomrule
\end{tabular}}

\end{table*}

\subsection{FE Models Show Dormant Coupling and Perturbation Contraction}

In the FE regime, the edge-prediction head is absent, so $M_X=0$ and the joint operator reduces to node-only dynamics. The coupling condition in Equation~\eqref{eq:de_spectral_witness} therefore cannot be activated. As shown in Figure~\ref{fig:h6}, injected perturbations rapidly decay to the numerical floor, indicating contractive FE dynamics. Error-Aware GWM also achieves substantially lower GEAF than ActionNode GWM across topologies, consistent with spectral regularization reducing rollout-error amplification.

\begin{figure}[!ht]
  
  \centering
  \includegraphics[width=0.85\linewidth]{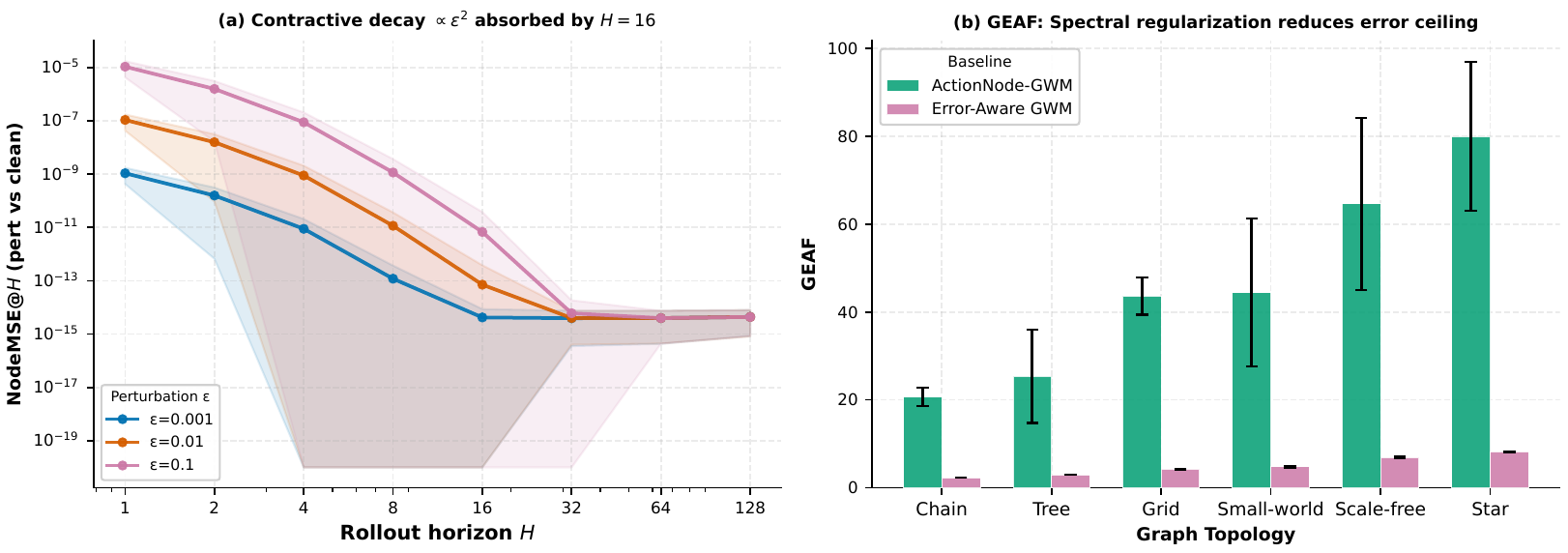}
  
  \caption{
    \textbf{FE-regime contraction and GEAF reduction.}
    Left: injected perturbations rapidly decay to the numerical floor. Right: Error-Aware GWM achieves lower GEAF than ActionNode GWM across topologies.
  }
  \label{fig:h6}
  
\end{figure}

\subsection{FE Perturbation Decomposition}

Figure~\ref{fig:app_fe_perturbation} further separates FE behavior under node-only, edge-only, and joint perturbations. Node-only perturbations have consistently negative GrowthSlope, showing that their effects decay over time. Edge-only perturbations remain near zero growth, while joint perturbations closely follow the edge-only case. This behavior is consistent with the FE condition $M_X=0$, under which node-edge cross-coupling is inactive.

\begin{figure}[!ht]
  
  \centering
  \includegraphics[width=0.8\linewidth]{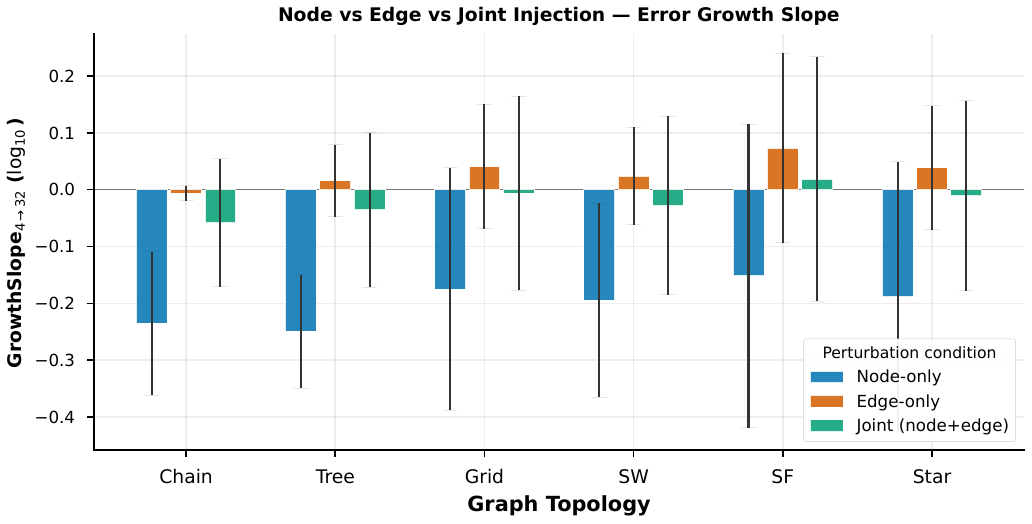}
  
  \caption{
    \textbf{Perturbation decomposition in the FE regime.}
  }
  \label{fig:app_fe_perturbation}
    
\end{figure}

\section{Conclusion}

We systematically studied rollout error in Graph World Models under fixed- and dynamic-edge settings, showing how topology, learned dynamics, and node-edge coupling govern long-horizon error propagation and planning regret. Our theoretical and empirical results characterize when graph-structured rollouts remain stable and when errors accumulate over long horizons. Motivated by these findings, we proposed Error-Aware GWM, which improves rollout stability across diverse graph topologies and heterogeneous agent-graph environments.


\subsection*{AI use statement}
AI tools were used for editorial polishing, including improvements to grammar, sentence structure, clarity, concision, and terminological consistency.

\bibliographystyle{iclr2027_conference}
\bibliography{references}


\clearpage
\appendix

\section{Related Work}
\label{sec:related}

\paragraph{World models for planning.}
World models have been widely studied in model-based reinforcement learning, where a learned dynamics model is used to simulate future trajectories for planning or policy learning. A central challenge is compounding error: small one-step model errors can accumulate over long rollouts and degrade decision quality. Asadi et al.~\cite{asadi2018lipschitz} analyze this effect through Lipschitz continuity of learned transition models. PETS uses probabilistic ensembles and trajectory sampling to propagate uncertainty during model-based planning~\cite{chua2018pets}. Janner et al.~\cite{janner2019trust} study model-based policy optimization with short branched rollouts to reduce error accumulation. Kakade and Langford~\cite{kakade2002approx} provide performance-difference tools for policy improvement. Related sequence-modeling work addresses exposure bias by training models on their own generated histories~\cite{bengio2015scheduled}. Recent world-model systems reduce rollout drift through retrieval-augmented context~\cite{chen2025learning}, hierarchical latent planning~\cite{zhang2026hierarchical}, Bayesian long-horizon planning~\cite{ni2025longhorizon}, modular dynamics experts~\cite{li2025prismatic}, skill-level abstraction~\cite{gurtler2025long}, operating-system simulation~\cite{rivard2025neuralos}, and controlled long-horizon evaluation benchmarks~\cite{li2025smallworlds}. These works motivate rollout-error analysis, but they do not model graph topology, message passing, or coupled node-edge prediction.

\paragraph{Graph neural networks and graph-structured prediction.}
Graph neural networks provide a standard framework for learning over graph-structured data by propagating information along edges. Interaction Networks and Graph Networks introduced object- and relation-centric neural dynamics for physical reasoning~\cite{battaglia2016interaction,battaglia2018relational}, and graph-network physics engines showed that such models can support inference and control in dynamical systems~\cite{sanchezgonzalez2018graph}. Graph convolutional networks aggregate normalized neighborhood information~\cite{kipf2017semi}. Message-passing neural networks generalize local edge-based aggregation~\cite{gilmer2017neural}. Graph transformers such as GPS combine message passing with global attention~\cite{rampasek2022recipe,ying2021graphormer}. Dynamic graph models and temporal graph networks learn over time-varying structure~\cite{rossi2020temporal,roland2022}, while graph generation methods model distributions over discrete graph structure~\cite{you2018graphrnn,vignac2022digress}. Prior theory shows that graph propagation is strongly affected by spectral properties of the adjacency or normalized adjacency matrix~\cite{oono2020graph}, and diffusion-based GNNs connect graph learning to stability of discretized diffusion dynamics~\cite{chamberlain2021grand}. Related work on graph rewiring studies how modifying topology can improve information flow and reduce harmful propagation effects~\cite{curvaturerewiring2024,piorf2025}. These results show that topology shapes prediction behavior, but they mainly study representation learning, graph generation, or one-step dynamics rather than autoregressive graph rollouts where prediction errors become future inputs.

\paragraph{Graph world models.}
Graph world models extend world modeling to graph-structured states. Feng et al.~\cite{feng2025gwm} introduce a GWM framework that represents multimodal observations, graph states, action nodes, and target nodes within a unified graph interface. Zhang et al.~\cite{zhang2020worldgraph} learn latent landmarks as a graph-structured abstraction for planning. Value Memory Graph constructs graph-structured memory for offline reinforcement learning~\cite{zhu2022valuememory}. Knowledge-graph world models represent textual environments through structured relational memory~\cite{kgwm2021}. Contrastive structured world models learn object-centric relational dynamics~\cite{kipf2019contrastive}. Dyn-O builds structured world models with object-centric representations~\cite{dyno2025}. These works mainly focus on representation, architecture, and task performance. Our work instead studies long-horizon error behavior: how topology affects rollout error, how fixed-edge and dynamic-edge regimes differ, and how node and edge prediction errors interact through a joint operator.

\paragraph{Agent graphs and skill-graph systems.}
Multi-agent and tool-using systems increasingly organize computation as graphs, including agent calling-trees, knowledge-graph memories, tool graphs, and skill graphs. API-Bank, ToolLLM/ToolBench, AgentBench, and AppWorld evaluate tool retrieval, API calling, multi-turn decision-making, and interactive digital tasks~\cite{li2023apibank,qin2024toolllm,liu2024agentbench,trivedi2024appworld}. AriGraph uses graph memory for LLM agents~\cite{anokhin2024arigraph}. Audited Skill-Graph Self-Improvement studies verifier-backed skill-graph maintenance~\cite{huang2025asg}. CASCADE studies cumulative agentic skill creation through autonomous development and evolution~\cite{cascade2025}. Reinforcement learning for self-improving agents with skill libraries studies policy learning over reusable skills~\cite{rl_sisl_2025}. GraSP compiles retrieved skills into a typed skill graph so failed subgraphs can be repaired locally rather than replanning the full trajectory~\cite{xia2026graspgraphstructuredskillcompositions}. These works show the practical value of graph-structured agent execution. Our work is complementary: we analyze how prediction errors accumulate when such graph states are rolled out over time, and we propose Error-Aware GWM to improve long-horizon stability in graph-based planning.

\section{Assumptions}
\label{sec:assumptions}

We use standard assumptions for finite-horizon model-based planning and graph neural networks. They require bounded graph states and one-step model errors, Lipschitz node and edge transitions, and bounded reward sensitivity. These conditions hold in our settings, where node identities remain fixed during each rollout.

\begin{assumption}[Bounded graph states]
\label{assump:compact_state}
For the graph state $G_t$ defined in Equation~\eqref{eq:gwm_state}, $|V|=N$, $X_t\in\mathbb{R}^{N\times D}$, and $A_t\in[0,1]^{N\times N}$. There exist $R_X,R_A<\infty$ such that
\begin{equation}
\|X_t\|_F\le R_X,\qquad \|A_t\|_F\le R_A
\end{equation}
for all rollout steps $t\in[0,T]$.
\end{assumption}

This holds in our experiments because node identities are fixed, node features are bounded, and adjacency entries lie in $[0,1]$. The analysis extends to growing graphs using a time-dependent operator $B_t$ with $\sup_t\rho(B_t)<\infty$.

\begin{assumption}[Bounded one-step model error]
\label{assump:onestep_error}
Let $F^\star=(F^{X\star},F^{A\star})$ denote the true transition and $F_\theta=(F_\theta^X,F_\theta^A)$ the learned transition from Equation~\eqref{eq:gwm_transition}. There exist $\epsilon_X,\epsilon_A\ge0$ such that
\begin{equation}
\begin{aligned}
\|F_\theta^X(X,A,a)-F^{X\star}(X,A,a)\|_F &\le\epsilon_X,\\
\|F_\theta^A(X,A,a)-F^{A\star}(X,A,a)\|_F &\le\epsilon_A.
\end{aligned}
\end{equation}
\end{assumption}

Here, $\epsilon_X$ and $\epsilon_A$ bound one-step node and edge prediction errors. The same analysis applies in expectation when only expected errors are available.

\begin{assumption}[Lipschitz message passing]
\label{assump:message_passing}
For the FE transition in Equation~\eqref{eq:fe_transition}, the activation $\sigma$ is coordinate-wise $L_\sigma$-Lipschitz with $L_\sigma\le1$, and the transition weights have finite spectral norms.
\end{assumption}

For multiple message-passing layers, the model dependence is captured by $\prod_\ell\|W_\ell\|_2$. This condition covers the GNN-based transition models considered in our analysis.

\begin{assumption}[Lipschitz dynamic edge update]
\label{assump:edge_update}
For the DE transition in Equation~\eqref{eq:de_transition}, the node and edge errors defined in Equation~\eqref{eq:node_edge_errors} satisfy
\begin{equation}
\begin{aligned}
e^X_{k+1} &\le L_X e^X_k+L_A e^A_k+\epsilon_X,\\
e^A_{k+1} &\le M_X e^X_k+M_A e^A_k+\epsilon_A,
\end{aligned}
\end{equation}
for finite constants $L_X,L_A,M_X,M_A\ge0$.
\end{assumption}

This condition covers differentiable edge predictors such as the bilinear predictor in Equation~\eqref{eq:edge_prediction}. In FE-GWM, edge prediction is absent and the recursion reduces to node error alone.

\begin{assumption}[Lipschitz reward]
\label{assump:reward}
For any graph states $G,G'$ and action $a$, the true reward satisfies
\begin{equation}
|R^\star(G,a)-R^\star(G',a)|\le L_R d_G(G,G'),
\end{equation}
and the learned reward satisfies
\begin{equation}
|R_\theta(G,a)-R^\star(G,a)|\le\epsilon_R,
\end{equation}
for constants $L_R,\epsilon_R\ge0$.
\end{assumption}

This assumption is required only for the planning-regret analysis in Lemma~\ref{lem:t4_regret}.

\begin{assumption}[Finite-horizon return]
\label{assump:discount}
For rollout horizon $H$, planning uses
\begin{equation}
J^\star=\sum_{k=0}^{H-1}\gamma^k r_{t+k},\qquad \gamma\in[0,1].
\end{equation}
\end{assumption}

The case $\gamma=1$ is allowed for finite $H$.

\begin{assumption}[Graph distance]
\label{assump:graph_distance}
Using the node and edge errors in Equation~\eqref{eq:node_edge_errors}, the graph distance satisfies
\begin{equation}
d_G(G_k,\hat G_k)\le \alpha e_k^X+\beta e_k^A,
\end{equation}
where $\alpha,\beta\ge0$.
\end{assumption}

This condition is natural for fixed node identities and also permits alternative graph distances that can be bounded by node and edge errors.

Together, Assumptions~\ref{assump:compact_state}--\ref{assump:graph_distance} support the FE and DE analyses. Lemma~\ref{lem:t1_fixed_edge} analyzes FE error amplification, Lemma~\ref{lem:t2_joint_operator} extends it to coupled node-edge errors, Lemma~\ref{lem:t3_geaf_proxy} relates $\GEAF$ to the joint operator, and Lemma~\ref{lem:t4_regret} connects rollout error to planning regret.

\section{Dataset and Simulator Details}
\label{app:dataset_details}

\subsection{Synthetic Graph Generation}

All synthetic graphs contain $N=50$ nodes and are evaluated using three outer seeds per topology. Scale-free graphs are generated by preferential attachment with $m=2$ edges per new node~\cite{barabasi1999emergence}. Small-world graphs use the Watts--Strogatz model with neighborhood size $k=4$ and rewiring probability $p=0.3$~\cite{watts1998collective}. Chain, tree, grid, star, and complete graphs are deterministic for fixed $N$. Each graph instance stores its adjacency matrix $A$, initial node features $X_0\in\mathbb{R}^{N\times8}$, and rollout trajectory $\{X_t,A_t\}_{t=1}^{T}$.

The seven families cover three broad structural regimes: low-connectivity or tree-like graphs (chain and tree), lattice or clustered graphs (grid and small-world), and hub-dominated or dense graphs (scale-free, star, and complete). Their spectral range, summarized in Table~\ref{tab:topo_stats}, provides controlled variation in the topology factor $\rho(A)$ appearing in $\GEAF$.

\subsection{Fixed-Edge Simulator}

The FE simulator follows the fixed-edge setting in Section~\ref{sec:t1_fixed_edge}. The adjacency matrix remains unchanged throughout each trajectory, while node features evolve as
\begin{equation}
X_t=f_{\rm true}(X_{t-1},A)+\varepsilon_t,\qquad
\varepsilon_t\sim\mathcal{N}(0,\sigma^2I),
\end{equation}
with $\sigma=0.01$. Node states $X_t\in\mathbb{R}^{N\times8}$ represent continuous state variables such as congestion or utilization. Each trajectory has length $T=50$, and prediction error is evaluated at horizons $H\in\{1,2,4,8,16,32\}$ using the autoregressive rollout in Equation~\eqref{eq:gwm_rollout}.

\subsection{Dynamic-Edge Simulator}

The DE simulator follows the dynamic-edge setting in Section~\ref{sec:t2_dynamic_edge}, with both $X_t$ and $A_t$ evolving during rollout. The edge-prediction head is
\begin{equation}
\hat A_t=\sigma(H_tQH_t^\top),
\label{eq:de_simulator_edge_head}
\end{equation}
where $H_t$ contains node representations, $Q$ is the learnable bilinear edge matrix, and $\sigma$ maps edge scores to $[0,1]$. We use $\|Q\|_2=1.0$ by default and rollout length $T=32$.

Unlike FE rollouts, DE rollouts activate both directions of node-edge coupling in the joint operator $B$ from Equation~\eqref{eq:t2_recursion}. Across our DE experiments, the estimated spectral radius lies in the range $\rho(B)=132$--$368$, compared with $2$--$50$ in the corresponding FE setting. This provides a substantially stronger amplification regime for testing the joint-error analysis in Lemma~\ref{lem:t2_joint_operator}.

\subsection{Agent Calling-Tree Testbed}

The agent calling-tree testbed represents heterogeneous task execution as a directed acyclic graph with approximately $22$--$30$ nodes per instance. It contains nine node types: \textit{planner}, \textit{validator}, \textit{executor}, \textit{checker}, \textit{aggregator}, \textit{reporter}, \textit{logger}, \textit{error\_handler}, and \textit{final\_answer}. Six relation types are used: \textit{calls}, \textit{validates}, \textit{reports}, \textit{errors}, \textit{triggers}, and \textit{logs}. Each node is represented by an 8-dimensional execution-status vector.

The primary evaluation metric is sink-node success rate, $\mathrm{sr}_{\rm sink}$, which measures whether the \textit{final\_answer} node reaches a successful terminal state. A random policy obtains $\mathrm{sr}_{\rm sink}=0.262$, while an oracle policy exceeds $0.40$, providing lower and upper calibration references for planning performance.

\subsection{Platform Skill-Graph Testbed}

The platform skill-graph testbed represents a heterogeneous skill-management environment with $N=40$ nodes and eight node types: \textit{skill}, \textit{validator}, \textit{adapter}, \textit{artifact}, \textit{failure\_case}, \textit{patch}, \textit{task}, and \textit{tool}. The graph captures dependencies among skills, validation components, adapters, artifacts, tools, and failure-recovery mechanisms, following the broader setting of tool/API and skill-based agent systems~\cite{li2023apibank,qin2024toolllm,liu2024agentbench,trivedi2024appworld,huang2025asg,cascade2025,xia2026graspgraphstructuredskillcompositions}.

The primary metric is skill success rate, $\mathrm{sr}_{\rm skill}$. Under a no-op policy, $\mathrm{sr}_{\rm skill}$ decreases from $1.0$ initially to approximately $0.317$ at $T=10$, confirming that skill quality degrades without maintenance actions. As an additional calibration, an R-GCN predictor~\cite{schlichtkrull2018modeling} evaluated against the ground-truth simulator obtains Pearson $r=0.231$ and MAE $=0.029$.

\section{Error-Aware GWM Details}
\label{app:error_aware_details}

This section gives the full formulation of Error-Aware GWM introduced in Section~\ref{sec:methods}. The method augments the standard GWM training objective with spectral regularization, rollout consistency, and critical-node weighting. For dynamic-edge models, it additionally regularizes the interaction between node and edge transitions.

\subsection{Base GWM Objective}

We use the GWM transition and rollout definitions in Equations~\eqref{eq:gwm_transition}--\eqref{eq:gwm_rollout}. In the FE setting, the adjacency matrix remains fixed and the model predicts only node states. In the DE setting, both node states and edges are predicted as defined in Equation~\eqref{eq:de_transition}. We denote the standard supervised prediction objective by $\mathcal{L}_{\rm pred}$; it contains the node-state prediction loss in FE-GWM and both node- and edge-prediction losses in DE-GWM.

Error-Aware GWM augments this objective as
\begin{equation}
\mathcal{L}_{\rm EA}
=
\mathcal{L}_{\rm pred}
+\lambda_{\rm spec}R_{\rm spec}
+\lambda_{\rm roll}R_{\rm roll}
+\lambda_{\rm crit}R_{\rm crit},
\label{eq:error_aware_objective_app}
\end{equation}
where $\lambda_{\rm spec},\lambda_{\rm roll},\lambda_{\rm crit}\ge0$ control the three regularization terms.

\subsection{Spectral and Rollout Regularization}

Lemma~\ref{lem:t1_fixed_edge} shows that FE rollout amplification depends on the product of transition-layer spectral norms through
$\GEAF(G)=\rho(A)\prod_\ell\|W_\ell\|_2$. We therefore regularize these norms using
\begin{equation}
R_{\rm spec}=\sum_{\ell}\|W_\ell\|_2^2.
\label{eq:r_spec}
\end{equation}
This reduces the model-dependent component of $\GEAF$ without modifying the graph topology. The motivation follows standard spectral control of neural-network Lipschitz constants~\cite{miyato2018spectral}. Its effect becomes particularly important when $\rho(A)$ is large, because topology and model amplification enter multiplicatively in the FE bound.

Spectral control limits amplification but does not directly optimize predictions under repeated model application. We therefore additionally train on autoregressive trajectories. For a set of rollout horizons $\mathcal{H}$,
\begin{equation}
R_{\rm roll}
=
\sum_{h\in\mathcal{H}}
\|\hat X_{t+h}^{(h)}-X_{t+h}\|_F^2,
\label{eq:r_roll}
\end{equation}
where $\hat X_{t+h}^{(h)}$ is obtained by applying the learned transition recursively for $h$ steps from the state at time $t$. Unlike the one-step objective, intermediate predictions are fed back into subsequent transitions. The model is therefore optimized under the prediction distribution encountered during rollout, reducing the mismatch between one-step training and long-horizon inference~\cite{bengio2015scheduled,janner2019trust}.

\subsection{Critical-Node Weighting}

Graph topology also makes the effect of prediction error nonuniform across nodes. An error at a structurally influential node can affect more downstream states through subsequent message passing than an error at a peripheral node. We account for this asymmetry with
\begin{equation}
R_{\rm crit}
=
\sum_{v\in V}
w_v\|\hat X_{t+1,v}-X_{t+1,v}\|_2^2,
\label{eq:r_crit}
\end{equation}
where $w_v\ge0$ is the normalized structural-importance weight of node $v$, with $\sum_{v\in V}w_v=1$. The weights can be computed from degree, PageRank, or betweenness centrality and are fixed during training. This term changes the allocation of prediction accuracy across the graph rather than introducing a new transition mechanism: errors on nodes with greater potential to propagate through the graph receive larger penalties.

\subsection{Dynamic-Edge Extension}

For DE-GWM, both node and edge errors contribute to rollout instability through the joint operator $B$ in Equation~\eqref{eq:t2_recursion}. We therefore apply the same node-state regularizers while including the edge-prediction loss in $\mathcal{L}_{\rm pred}$. The edge transition uses the bilinear predictor defined in Equation~\eqref{eq:edge_prediction},
\begin{equation}
\hat A_{k+1,ij}=\sigma(h_{k,i}^{\top}Qh_{k,j}),
\end{equation}
where $Q$ is the learnable edge-prediction matrix.

To control the interaction between node and edge transitions, we additionally use
\begin{equation}
R_g=\frac{1}{2}\|W\|_2\|Q\|_2,
\label{eq:de_graph_regularizer}
\end{equation}
where $W$ denotes the node-transition weight matrix and $Q$ the edge-prediction weight matrix. The complete DE objective is therefore
\begin{equation}
\mathcal{L}_{\rm EA}^{\rm DE}
=
\mathcal{L}_{\rm pred}^{\rm DE}
+\lambda_{\rm spec}R_{\rm spec}
+\lambda_{\rm roll}R_{\rm roll}
+\lambda_{\rm crit}R_{\rm crit}
+\lambda_gR_g,
\label{eq:error_aware_de_objective}
\end{equation}
where $\lambda_g\ge0$ controls the DE-specific graph regularizer. This term targets the cross-coupling responsible for the off-diagonal components of $B$, complementing spectral regularization of the node transition.

Together, these components address three distinct mechanisms identified by our analysis: $R_{\rm spec}$ limits transition amplification, $R_{\rm roll}$ reduces error accumulation under autoregressive prediction, and $R_{\rm crit}$ concentrates predictive accuracy on structurally influential nodes. The DE extension further controls node-edge coupling when the graph structure itself evolves.

\section{RQ1: Does Topology Predict Rollout-Error Growth?}

We further investigate \textbf{RQ1} through the following additional analyses.

\subsection{Rollout Amplification vs.\ Over-Smoothing}
\label{app:oversmoothing_vs_geaf}

We distinguish GWM rollout amplification from classical GNN over-smoothing. Over-smoothing results from repeated normalized propagation that drives node representations toward a low-variation regime, whereas GWM rollout error is governed by the autoregressive amplification factor $\widehat{\mathrm{GEAF}}(\ell)=\rho(A)^\ell c^\ell$. Figure~\ref{fig:app_oversmoothing_geaf} shows that the normalized propagation operator remains bounded while $\widehat{\mathrm{GEAF}}(\ell)$ increases with message-passing depth and graph spectral radius. Thus, the two capture distinct mechanisms: representation homogenization and autoregressive error amplification.

\begin{figure}[!ht]
  \centering
  \includegraphics[width=0.8\linewidth]{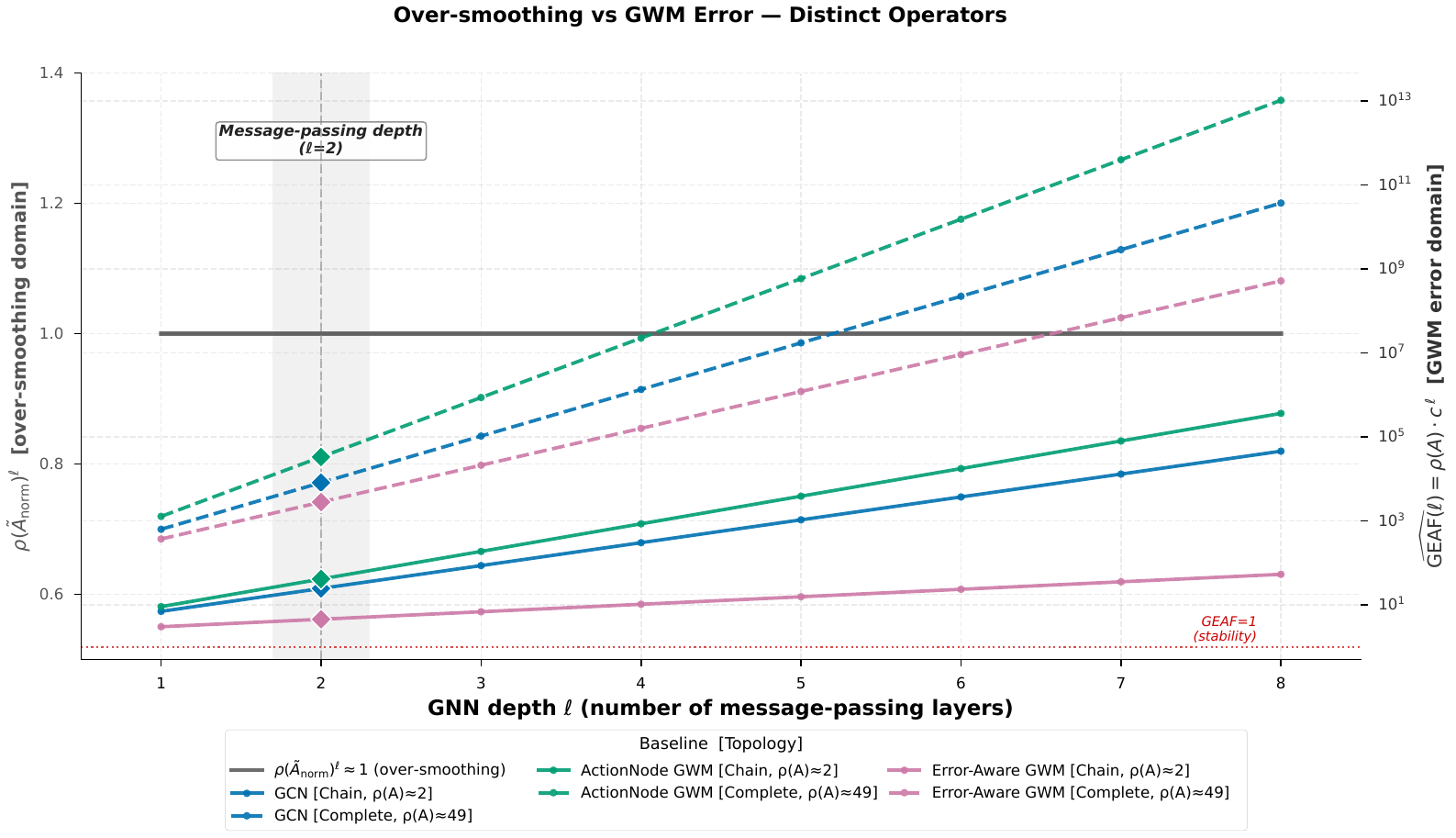}
  
  \caption{\textbf{Rollout amplification vs.\ over-smoothing.}
  The normalized propagation operator remains bounded, whereas $\widehat{\mathrm{GEAF}}(\ell)$ increases with message-passing depth and graph spectral radius, distinguishing representation homogenization from autoregressive error amplification.}
  \label{fig:app_oversmoothing_geaf}
\end{figure}

\subsection{Global Spectral Amplification vs.\ Local Hub Failure}
\label{app:global_local_fe_diagnostics}

We distinguish global spectral amplification from local hub-mediated failure. Figure~\ref{fig:app_global_local} shows that complete graphs have the largest $\widehat{\mathrm{GEAF}}$, whereas star graphs exhibit the largest affected-node fraction under hub injection despite having lower $\widehat{\mathrm{GEAF}}$. Thus, $\rho(A)$ captures global spectral amplification, while hub concentration governs the spread of localized perturbations.

\begin{figure}[!ht]
  \centering
  \includegraphics[width=0.8\linewidth]{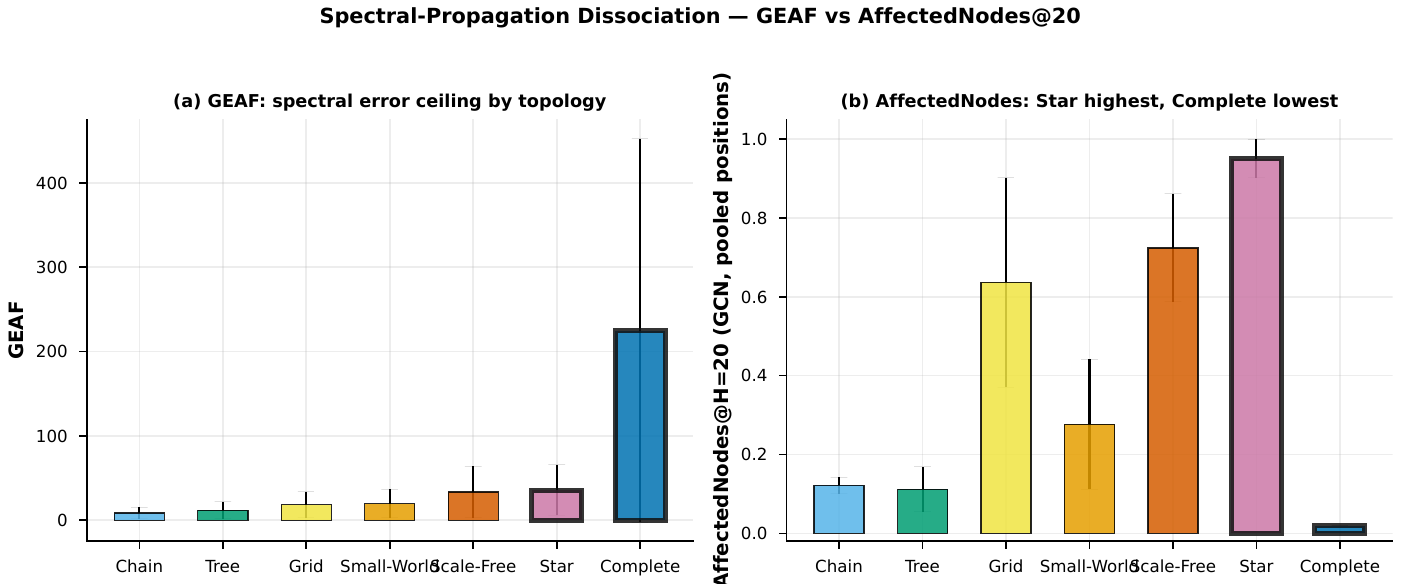}
  
  \caption{\textbf{Global spectral amplification vs.\ local hub failure.}
  Left: complete graphs have the largest $\widehat{\mathrm{GEAF}}$. Right: star graphs exhibit the largest affected-node fraction under hub injection, showing that global spectral amplification and local hub-mediated spread are distinct.}
  \label{fig:app_global_local}
\end{figure}

\subsection{Edge-Density Sweep}
\label{app:edge_density}

We vary graph density using Erd\H{o}s--R\'enyi graphs with edge probability $p\in\{0.02,0.05,0.10,0.20\}$ and compare the vanilla GCN world model, ActionNode GWM, and Error-Aware GWM over three seeds. Table~\ref{tab:sup_exp22} shows that increasing density raises $\rho(A)$ and substantially increases NodeMSE@32 for vanilla GCN, while ActionNode GWM and Error-Aware GWM remain near the low-error floor. This controlled sweep further supports the role of topology-dependent spectral amplification in rollout stability.

\begin{table}[!ht]
\centering
\caption{\textbf{Effect of edge density on long-horizon rollout error.}
NodeMSE@32 is reported over three random seeds as the edge probability $p$ and corresponding spectral radius $\rho(A)$ increase.}
\label{tab:sup_exp22}
\resizebox{0.7\linewidth}{!}{
\begin{tabular}{@{}ccccc@{}}
\toprule
$p$ & $\rho(A)$ & \multicolumn{3}{c}{NodeMSE@32} \\
\cmidrule(lr){3-5}
& & GCN & ActionNode GWM & Error-Aware GWM \\
\midrule
0.02 & 2.541  & $2.220\times10^{-1}$ & $5.630\times10^{-4}$ & $5.571\times10^{-4}$ \\
0.05 & 3.684  & $1.186\times10^{2}$  & $5.320\times10^{-4}$ & $5.760\times10^{-4}$ \\
0.10 & 5.926  & $3.597\times10^{15}$ & $5.531\times10^{-4}$ & $4.208\times10^{-4}$ \\
0.20 & 10.478 & $2.193\times10^{31}$ & $6.045\times10^{-4}$ & $6.470\times10^{-4}$ \\
\bottomrule
\end{tabular}}
\end{table}

\section{RQ3: Does Error-Aware GWM Improve Long-Horizon Stability?}
\label{sec:RQ2_error_aware}

We evaluate whether Error-Aware GWM stabilizes long-horizon rollouts without sacrificing accuracy. This RQ compares the full objective with vanilla GCN rollouts and simpler stabilization methods such as gradient clipping, weight decay, and spectral-norm projection. We also test robustness under graph rewiring, noisy observations, and graph-size scaling.

\subsection{Error-Aware GWM Prevents Divergence While Preserving Accuracy}

Table~\ref{tab:h5} presents the main FE comparison. On high-risk scale-free and star topologies, the vanilla GCN world model shows rapidly increasing error as the horizon grows, while Error-Aware GWM stays near the low-error floor. Across all paired topology-seed cells, the full objective eliminates the divergences observed in the vanilla GCN setting. This result supports the design of Error-Aware GWM: spectral control is most effective when combined with rollout consistency and critical-node weighting, allowing the model to remain both stable and accurate.

\begin{table}[!ht]
\centering
\caption{\textbf{Long-horizon rollout stability on scale-free and star graphs.}
NodeMSE is reported at increasing rollout horizons for GCN and Error-Aware GWM. Values are means over three random seeds; $\infty$ denotes a non-finite recorded result.}
\label{tab:h5}
\resizebox{0.5\linewidth}{!}{
\begin{tabular}{@{}lccc@{}}
\toprule
Topology & $H$ & GCN & Error-Aware GWM \\
\midrule
Scale-free
& 1  & $0.5884$              & $2.366\times10^{-3}$ \\
& 4  & $0.6364$              & $2.377\times10^{-4}$ \\
& 8  & $247.5$               & $2.577\times10^{-4}$ \\
& 16 & $3.512\times10^{9}$   & $2.623\times10^{-4}$ \\
& 32 & $4.439\times10^{23}$  & $2.853\times10^{-4}$ \\
\midrule
Star
& 1  & $0.1147$              & $1.460\times10^{-3}$ \\
& 4  & $23.68$               & $1.562\times10^{-4}$ \\
& 8  & $9.955\times10^{5}$   & $1.501\times10^{-4}$ \\
& 16 & $4.116\times10^{15}$  & $1.564\times10^{-4}$ \\
& 32 & $\infty$              & $1.541\times10^{-4}$ \\
\bottomrule
\end{tabular}}
\end{table}

\subsection{Ablation: Stability Alone Is Not Enough}

Table~\ref{tab:b6_ablation} shows that simple stabilization methods reduce divergence, but they neither eliminate it nor maintain prediction accuracy. Weight decay and spectral-norm projection reduce divergence to 4/21 cells, yet their median NodeMSE@32 remains orders of magnitude larger than that of Error-Aware GWM. Spectral-norm projection has $1690\times$ higher error than Error-Aware GWM, while weight decay approaches a nearly constant predictor. The advantage of Error-Aware GWM is therefore not only lower divergence, but the combination of stability and prediction accuracy.

\begin{table*}[!ht]
\centering
\caption{
Regularization ablation over 21 FE cells. Simple stabilization methods reduce divergence but leave large long-horizon prediction errors; Error-Aware GWM is the only variant that combines zero divergence with low NodeMSE@32, showing that stability must be maintained together with prediction accuracy.
}
\label{tab:b6_ablation}
\small
\resizebox{\textwidth}{!}{
\begin{tabular}{llcrl}
\toprule
\hline
Variant & Regularizer & Diverged & Median NodeMSE@32 & Ratio vs.\ Error-Aware GWM \\
\midrule
Vanilla GCN world model & none & 10/21 & 0.995 & $6934\times$ \\
GCN + gradient clipping & gradient clipping & 8/21 & 0.987 & $5790\times$ \\
GCN + weight decay & weight decay & 4/21 & $1.091$ & $4172\times$ \\
GCN + spectral-norm projection & spectral norm projection & 4/21 & 0.364 & $1690\times$ \\
\midrule
\textbf{Error-Aware GWM} & $R_\text{spec}+R_\text{roll}+R_\text{crit}$ & \textbf{0/21} & $\mathbf{2.67 \times 10^{-4}}$ & \textbf{1$\times$} \\
\hline
\bottomrule
\end{tabular}
}
\end{table*}

\subsection{Robustness to Topology Perturbation and Graph Size}

\paragraph{Graph Rewiring Robustness.}
We test whether rollout stability is robust to changes in graph topology by applying six rewiring operations to scale-free graphs: identity, PageRank-based edge removal, random edge addition, random edge flipping, sparsification, and densification. We compare the vanilla GCN world model, ActionNode GWM, and Error-Aware GWM over three seeds. Figure~\ref{fig:sup_exp11} shows that vanilla GCN remains in a high-error regime across all rewiring operations. ActionNode GWM is stable under most changes but has higher error under densification and random edge addition. In contrast, Error-Aware GWM remains stable across all rewiring operations, with NodeMSE@$H=20$ between approximately $2.0\times10^{-4}$ and $5.0\times10^{-4}$. These results show that Error-Aware GWM remains stable under topology perturbations, consistent with the role of spectral regularization in controlling topology-dependent error amplification.

\begin{figure}[!ht]
  \centering
  \includegraphics[width=0.8\linewidth]{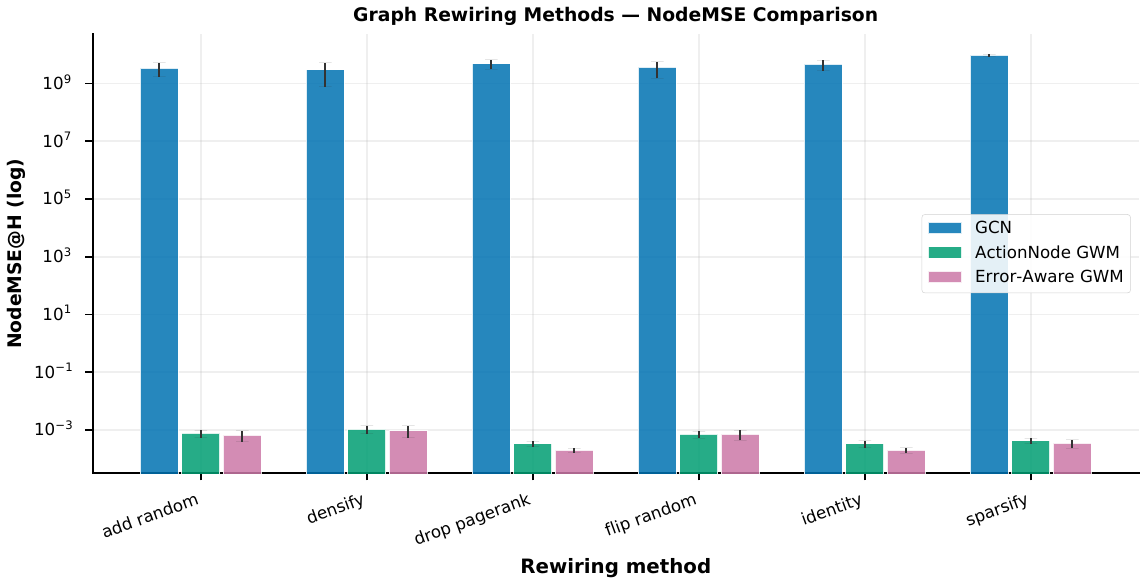}
  
  \caption{\textbf{Robustness to graph rewiring.}
  NodeMSE@$H=20$ under six rewiring operations on scale-free graphs across three world-model variants and three seeds. Error-Aware GWM remains in the low-error regime across all rewiring operations, while vanilla GCN remains unstable and ActionNode GWM is more sensitive to increased connectivity.}
  \label{fig:sup_exp11}
  
\end{figure}

\paragraph{Graph Size Scaling.}
We next test whether the observed stability patterns persist as graph size increases. We evaluate $N\in\{20,50,100,200,500\}$ across five topologies, six baselines, and three seeds. Table~\ref{tab:sup_exp21} reports $\widehat{\mathrm{GEAF}}$ and NodeMSE@32 as functions of $N$. ActionNode GWM, GPS, and Error-Aware GWM remain near the low-error regime across graph sizes, whereas vanilla GCN continues to produce large rollout errors in high-risk settings. MPNN encounters out-of-memory failures at larger graph sizes because of its dense message tensors. Thus, the stability patterns observed at $N=50$ continue across a substantially wider range of graph sizes.

\begin{table}[!ht]
\centering
\caption{\textbf{Scaling with graph size.}
Results are aggregated over five graph topologies and three random seeds (15 runs per available entry). $\widehat{\mathrm{GEAF}}$ is reported as mean $\pm$ sample standard deviation, and NodeMSE@32 as the mean. A dash indicates an unavailable result.}
\label{tab:sup_exp21}
\resizebox{0.7\linewidth}{!}{
\begin{tabular}{@{}clcc@{}}
\toprule
$N$ & Method & $\widehat{\mathrm{GEAF}}$ & NodeMSE@32 ($\times10^{-4}$) \\
\midrule
20  & MPNN            & $12.258 \pm 4.150$  & 5.5925 \\
    & GPS             & $11.913 \pm 3.977$  & 7.8996 \\
    & ActionNode GWM  & $14.296 \pm 4.794$  & 7.2517 \\
    & Error-Aware GWM & $3.681 \pm 1.270$   & 5.9736 \\
\midrule
50  & MPNN            & $13.982 \pm 5.542$  & 3.6540 \\
    & GPS             & $13.116 \pm 5.152$  & 5.5694 \\
    & ActionNode GWM  & $27.174 \pm 12.419$ & 4.5876 \\
    & Error-Aware GWM & $4.138 \pm 1.650$   & 4.2333 \\
\midrule
100 & MPNN            & $14.319 \pm 6.659$  & 4.0767 \\
    & GPS             & $14.141 \pm 6.426$  & 9.1332 \\
    & ActionNode GWM  & $17.347 \pm 7.741$  & 5.2006 \\
    & Error-Aware GWM & $4.290 \pm 1.952$   & 3.8934 \\
\midrule
200 & MPNN            & --                  & -- \\
    & GPS             & $14.893 \pm 7.428$  & 8.6677 \\
    & ActionNode GWM  & $16.228 \pm 8.402$  & 5.9376 \\
    & Error-Aware GWM & $4.548 \pm 2.366$   & 3.9875 \\
\midrule
500 & MPNN            & --                  & -- \\
    & GPS             & $15.859 \pm 9.790$  & 10.1959 \\
    & ActionNode GWM  & $17.493 \pm 10.158$ & 6.3460 \\
    & Error-Aware GWM & $4.876 \pm 2.994$   & 4.6260 \\
\bottomrule
\end{tabular}}
\end{table}

\section{RQ4: How Do GWMs Behave on Heterogeneous Agent Graphs?}
\label{sec:RQ4_agent_graphs}

We evaluate GWMs on heterogeneous agent and skill graphs, where nodes represent different roles such as planners, executors, validators, tools, and failure handlers. We study whether R-GCN rollouts match the ground-truth simulators, whether rollout failures depend on node roles, whether targeted correction can reduce errors, and how task success changes with rollout horizon.

\subsection{Calibration Against Heterogeneous Simulators}
\label{sec:rgcn_calibration}

We first compare R-GCN rollout predictions with the ground-truth simulators on the agent calling-tree and platform skill-graph testbeds. As shown in Figure~\ref{fig:f8}, on the agent calling-tree testbed, the R-GCN mean is $0.442$ versus a ground-truth mean of $0.437$, with Pearson $r=0.311$ and MAE $=0.051$. On the skill-graph testbed, the corresponding means are $0.387$ and $0.406$, with Pearson $r=0.231$ and MAE $=0.029$. The predictions are close to the simulator outputs in mean and absolute error, but the instance-level correlations remain limited. We therefore use these R-GCN models to study rollout behavior on heterogeneous graphs rather than as precise replacements for the ground-truth simulators.

\begin{figure}[!ht]
  \centering
  \includegraphics[width=0.8\linewidth]{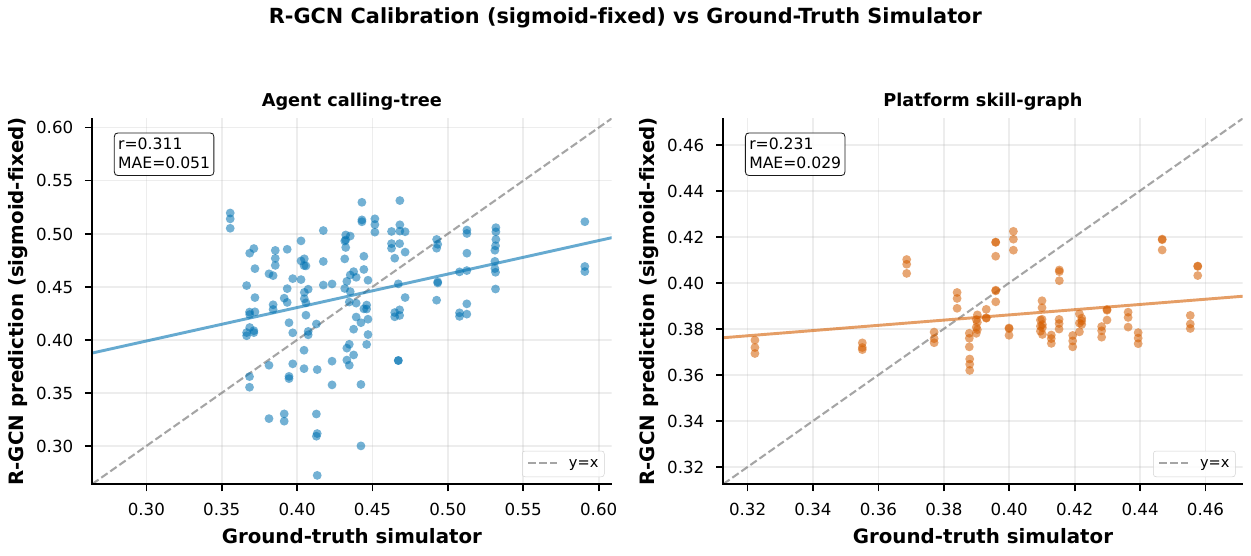}
  \caption{
    \textbf{R-GCN calibration against heterogeneous simulators.}
    Left: agent calling-tree testbed. Right: platform skill-graph testbed. R-GCN predictions are close to the simulator outputs in mean and absolute error, while instance-level correlations remain limited.
  }
  \label{fig:f8}
\end{figure}

\subsection{Role-Dependent Failure Sensitivity}
\label{app:exp9_subgraph_masking}

We next test whether different functional roles contribute differently to rollout quality. We mask role-specific subgraphs in the agent calling-tree testbed and measure the resulting success-rate drop and trajectory-level NodeMSE. Figure~\ref{fig:sup_exp9} shows that masking the full graph causes the largest degradation. Among partial masks, removing hub- and validator-related components has a larger effect than removing leaf nodes, while leaf masking has little effect. This shows that structurally or functionally important nodes contribute more strongly to rollout quality.

\begin{figure}[!ht]
  \centering
  \includegraphics[width=0.8\linewidth]{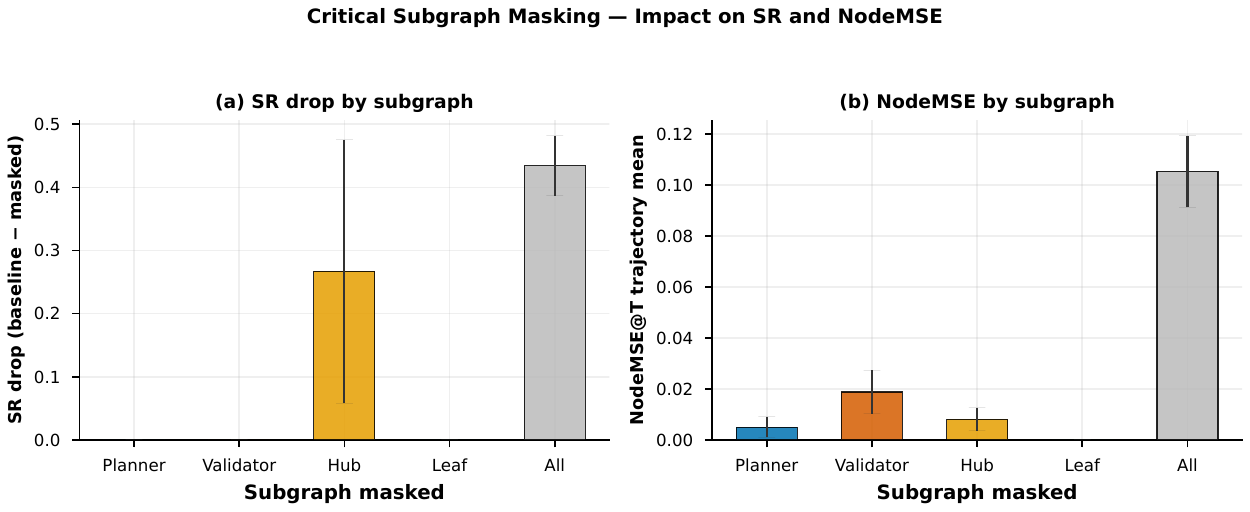}
  \caption{
    \textbf{Role-specific subgraph masking on the agent calling-tree testbed.}
    Left: success-rate drop after masking different subgraphs. Right: NodeMSE@$T$ under the same masking conditions. Full-graph masking causes the largest degradation, while leaf masking has little effect.
  }
  \label{fig:sup_exp9}
\end{figure}

We further study how errors propagate when they originate from different roles. We inject perturbations into planner, validator, hub, bridge, and random nodes across 20 instances. Figure~\ref{fig:sup_exp15} reports failure propagation depth (FPD) and the number of error nodes at the final timestep. Planner and validator injections have a median FPD of about 3, while bridge and hub injections have smaller median FPD. Random injection has the largest median FPD of 4. Final error-node counts are similar on average, but their distributions differ across roles. These results show that error propagation depends on functional role rather than graph position alone.

\begin{figure}[!ht]
  \centering
  \includegraphics[width=0.8\linewidth]{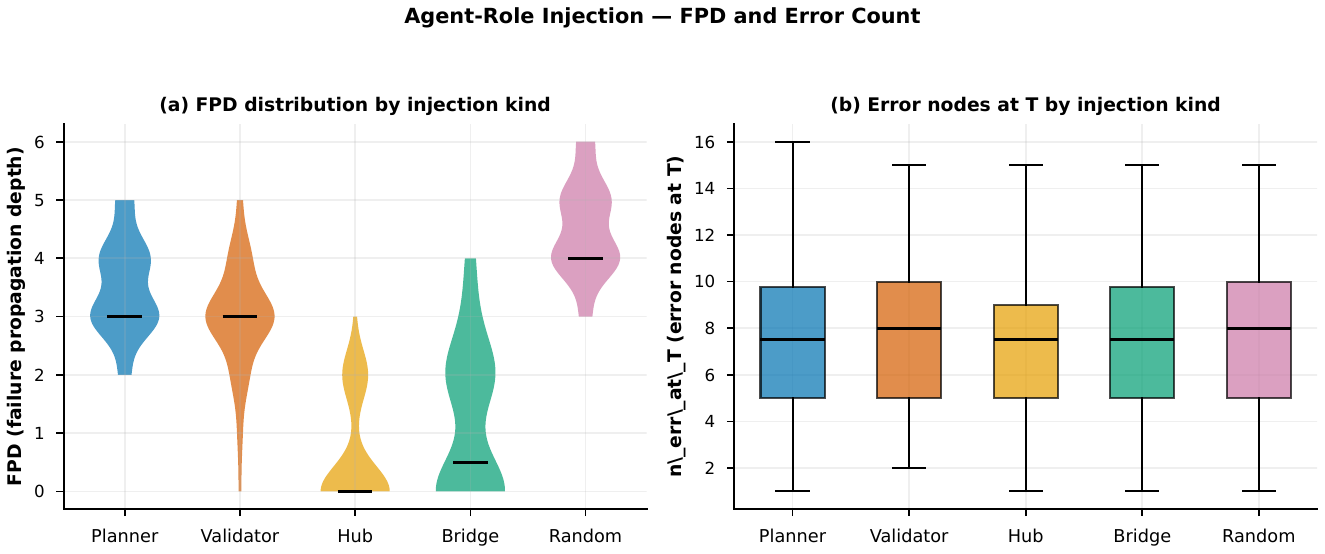}
  \caption{
    \textbf{Error injection across agent roles.}
    Left: failure propagation depth by injection role. Right: number of error nodes at the final timestep. Different roles produce different propagation patterns in the calling tree.
  }
  \label{fig:sup_exp15}
\end{figure}

\subsection{Correction Policies}

We next test whether a local GEAF score can guide node-level correction. In the FE setting,
\begin{equation}
\GEAF_{\rm local}(v)=\mathrm{degree}(v)\prod_\ell\|W_\ell\|_2.
\label{eq:geaf_local_degree}
\end{equation}
Because $\prod_\ell\|W_\ell\|_2$ is constant across nodes for a fixed model, $\GEAF_{\rm local}(v)$ gives the same node ranking as degree centrality. Figure~\ref{fig:f7} confirms this result on the agent calling-tree testbed: GEAF-proxy and degree correction both achieve a mean error reduction of 12.9\%, outperforming random correction but remaining below the oracle policy at 22.4\%. Thus, local GEAF provides a useful structural correction rule, but in the FE setting its ranking is equivalent to degree centrality.

We also compare correction policies more broadly across six baselines, six topologies, and three seeds under the same 10\% correction budget. Figure~\ref{fig:sup_exp10} shows that oracle correction provides the largest NodeMSE reduction. Among practical policies, uncertainty-based correction gives the largest average improvement over random correction, while GEAF-proxy correction improves over random but matches degree-based correction. This is consistent with Equation~\eqref{eq:geaf_local_degree}: for fixed model weights, local GEAF and degree induce the same ranking.

\begin{figure}[!ht]
  \centering
  \includegraphics[width=0.8\linewidth]{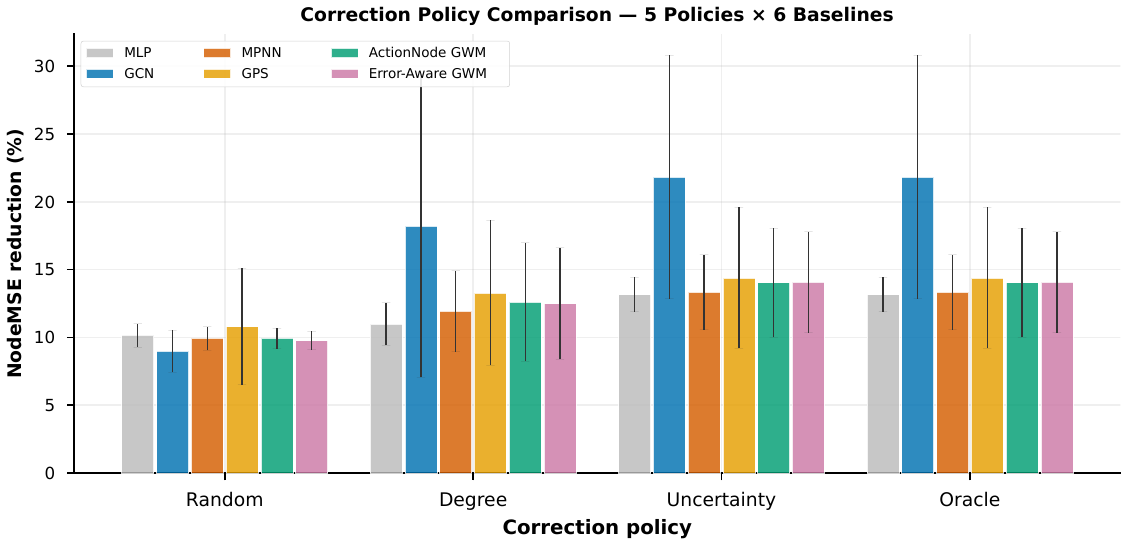}
  \caption{
    \textbf{Node correction policy comparison.}
    NodeMSE reduction under a 10\% correction budget across six baselines, six topologies, and three seeds. Oracle correction gives the largest reduction, uncertainty is the strongest practical policy, and GEAF-proxy matches degree-based correction in the FE setting.
  }
  \label{fig:sup_exp10}
\end{figure}

Finally, we evaluate the same correction idea directly on the agent calling-tree testbed. Figure~\ref{fig:sup_exp16} compares random, degree, GEAF-proxy, and oracle correction under a 10\% node-correction budget. Oracle correction gives the largest reduction in the mean error flag, while GEAF-proxy and degree remain tied because they select the same nodes. Together, these experiments show that local GEAF provides a useful non-random correction rule in FE agent graphs, but does not add information beyond degree centrality in this setting.

\begin{figure*}[!ht]
  \centering
  \resizebox{0.8\textwidth}{!}{
  \begin{minipage}{\textwidth}
    \centering
    \begin{subfigure}[t]{0.49\textwidth}
      \centering
      \includegraphics[width=\linewidth]{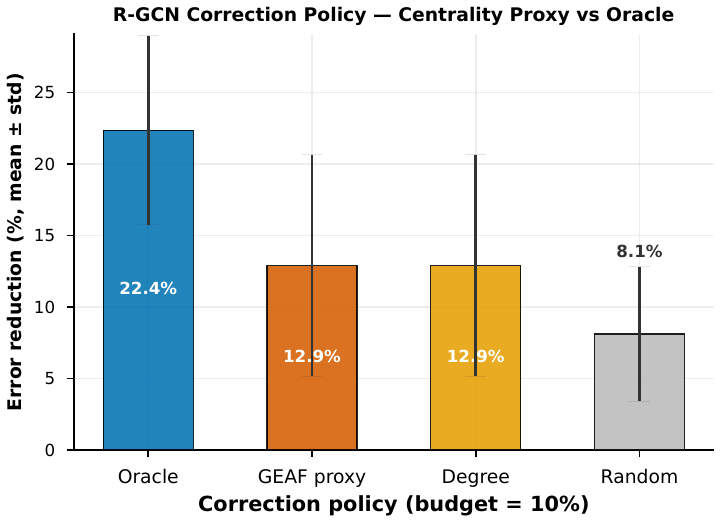}
      \caption{Local GEAF correction.}
      \label{fig:f7}
    \end{subfigure}
    \hfill
    \begin{subfigure}[t]{0.49\textwidth}
      \centering
      \includegraphics[width=\linewidth]{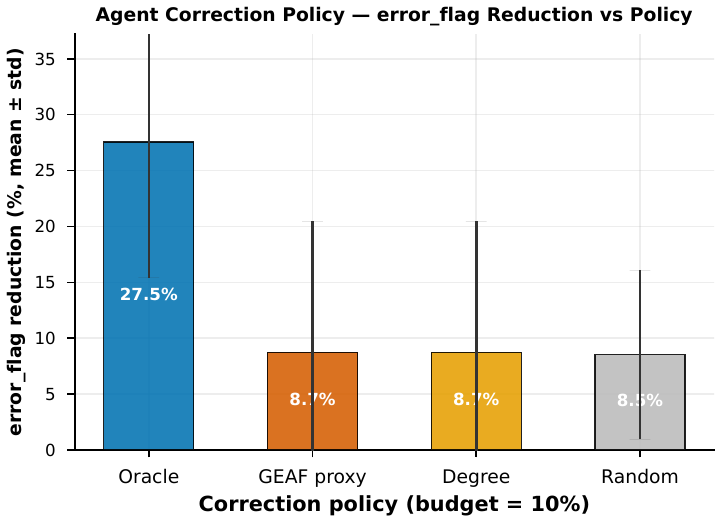}
      \caption{Agent-level correction.}
      \label{fig:sup_exp16}
    \end{subfigure}
  \end{minipage}
  }
  \caption{\textbf{Correction policies on heterogeneous agent graphs.}
  (a) Local GEAF and degree-based correction give the same ranking in the FE setting because the local GEAF score is degree multiplied by a model-wide constant.
  (b) Under a 10\% node-correction budget, oracle correction gives the largest reduction in mean error flag, while GEAF-proxy and degree correction remain tied because they induce the same FE node ranking.}
  \label{fig:correction_policies}
\end{figure*}

\subsection{Long-Horizon Agent Task Completion}
\label{sec:rw_agent}

We finally evaluate task completion on the agent calling-tree and platform skill-graph testbeds. We pool 160 instances per testbed and report task success rate (SR) at horizons $H\in\{1,5,10,20\}$. Results are averaged over six testbed-seed settings for each method and horizon. As shown in Table~\ref{tab:exp_C1_agent}, all methods achieve SR around $0.48$--$0.49$ at $H=1$, which decreases to approximately $0.41$--$0.42$ at longer horizons. Error-Aware GWM remains competitive but does not remove this decline, showing that long-horizon task completion remains difficult on heterogeneous agent graphs.

\begin{table}[!ht]
\centering
\caption{
Agent task completion across rollout horizons. All methods decline from approximately $0.49$ at $H=1$ to approximately $0.42$ for $H\ge5$, showing the difficulty of long-horizon planning on heterogeneous agent graphs.
}
\label{tab:exp_C1_agent}
\small
\resizebox{0.9\columnwidth}{!}{
\begin{tabular}{lcccc}
\toprule
Baseline & SR@1 & SR@5 & SR@10 & SR@20 \\
\midrule
MLP world model & 0.477 $\pm$ 0.039 & 0.420 $\pm$ 0.029 & 0.422 $\pm$ 0.033 & 0.423 $\pm$ 0.033 \\
GCN world model & 0.487 $\pm$ 0.069 & 0.415 $\pm$ 0.025 & 0.415 $\pm$ 0.025 & 0.416 $\pm$ 0.025 \\
MPNN world model & 0.492 $\pm$ 0.084 & 0.411 $\pm$ 0.028 & 0.411 $\pm$ 0.026 & 0.413 $\pm$ 0.023 \\
Error-Aware GWM (ours) & 0.491 $\pm$ 0.080 & 0.420 $\pm$ 0.030 & 0.417 $\pm$ 0.028 & 0.417 $\pm$ 0.025 \\
\bottomrule
\end{tabular}}
\end{table}

\section{RQ5: What is the Real-World Scope of GWMs}

\label{sec:RQ5_realworld}

Finally, we evaluate GWMs on standard real-world graph benchmarks to examine their practical scope relative to task-specific graph models. We consider static node classification and temporal link prediction, where GWMs are not specifically designed to replace specialized architectures.

\subsection{Static Node Classification}
\label{sec:realworld}

We evaluate Cora and Citeseer~\cite{yang2016revisiting} using their standard public splits. Cora contains 2,708 nodes, 5,278 edges, and 7 classes, while Citeseer contains 3,327 nodes, 4,732 edges, and 6 classes. We compare FE-GWM and Error-Aware GWM with GCN~\cite{kipf2017semi}, GAT~\cite{velivckovic2017graph}, and GraphSAGE~\cite{hamilton2017graphsage}. All models are trained for 200 epochs over three seeds.

\begin{table}[!ht]
\centering
\caption{Test accuracy (\%) on Cora and Citeseer, reported as mean $\pm$ standard deviation over three seeds.}

\label{tab:exp_A_node_class}
\small
\resizebox{0.7\columnwidth}{!}{
\begin{tabular}{lcc}
\toprule
Method & Cora & Citeseer \\
\midrule
GCN~\cite{kipf2017semi} & \textbf{80.6 $\pm$ 1.0} & 65.6 $\pm$ 2.4 \\
GAT~\cite{velivckovic2017graph} & 79.3 $\pm$ 0.4 & \textbf{68.1 $\pm$ 0.8} \\
GraphSAGE~\cite{hamilton2017graphsage} & 78.8 $\pm$ 0.6 & 66.5 $\pm$ 1.1 \\
FE-GWM (ours) & 73.5 $\pm$ 1.6 & 66.2 $\pm$ 2.7 \\
Error-Aware GWM (ours) & 75.0 $\pm$ 1.6 & 66.9 $\pm$ 1.8 \\
\bottomrule
\end{tabular}}

\end{table}

Table~\ref{tab:exp_A_node_class} shows that specialized classifiers remain stronger on Cora, whereas the gap is smaller on Citeseer. Error-Aware GWM improves over FE-GWM on both datasets, indicating a modest benefit from error-aware regularization even in the static setting.

\subsection{Temporal Link Prediction}
\label{sec:rw_temporal}

We evaluate Bitcoin-Alpha~\citep{kumar2016edge}, a temporal trust network with 3,783 nodes and 24,186 edges, using 60-day windows. We compare GWM variants with a JODIE-style temporal model~\citep{kumar2019predicting} and Static GCN~\citep{kipf2017semi}, reporting AP, AUC, and MRR over three seeds.

\begin{table}[!ht]
\centering
\caption{Temporal link prediction on Bitcoin-Alpha, reported as mean $\pm$ standard deviation over three seeds.}

\label{tab:exp_B_temporal}
\small
\resizebox{0.85\columnwidth}{!}{
\begin{tabular}{lccc}
\toprule
Method & AP & AUC & MRR \\
\midrule
JODIE-style~\citep{kumar2019predicting} & 0.772 $\pm$ 0.009 & 0.646 $\pm$ 0.009 & \textbf{0.261 $\pm$ 0.104} \\
Static GCN~\citep{kipf2017semi} & 0.815 $\pm$ 0.014 & \textbf{0.774 $\pm$ 0.030} & 0.109 $\pm$ 0.034 \\
FE-GWM (ours) & \textbf{0.820 $\pm$ 0.016} & 0.740 $\pm$ 0.030 & 0.166 $\pm$ 0.056 \\
DE-GWM (ours) & 0.805 $\pm$ 0.004 & 0.717 $\pm$ 0.009 & 0.122 $\pm$ 0.037 \\
Error-Aware DE-GWM (ours) & 0.805 $\pm$ 0.006 & 0.713 $\pm$ 0.009 & 0.134 $\pm$ 0.041 \\
\bottomrule
\end{tabular}}

\end{table}

Table~\ref{tab:exp_B_temporal} shows a mixed result: FE-GWM achieves the highest AP, Static GCN the highest AUC, and JODIE-style the highest MRR. DE-GWM does not consistently improve over FE-GWM, indicating that evolving-edge modeling is most useful when edge dynamics provide sufficiently predictive rollout signal.

\section{Additional Experiments}
\label{app:all_exps}

This section provides additional analyses of seed stability, multi-hop error propagation, message-passing architectures, and temporal memory.

\subsection{Seed Stability of Error Metrics}
\label{app:seed_stability}

We first examine whether the observed rollout-error patterns depend strongly on random seeds. Figure~\ref{fig:app_seed_stability} reports seed-level distributions of NodeMSE@32 and GrowthSlope across the seven topology families. The distributions remain stable across seeds, with particularly small variation on deterministic topologies. This supports the three-seed protocol used in the main experiments and indicates that the observed topology-dependent trends are not driven by isolated graph instances.

\begin{figure}[!ht]
  \centering
  \includegraphics[width=\linewidth]{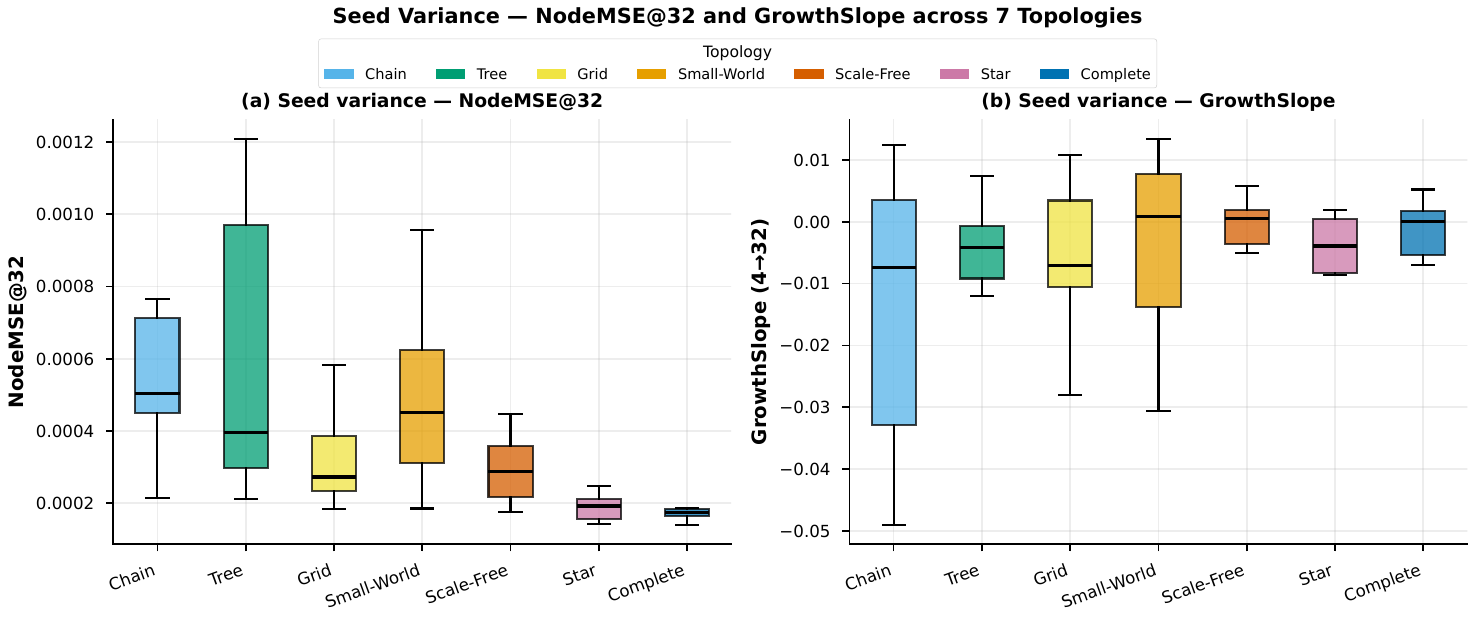}
  \caption{
  \textbf{Seed stability across topologies.}
  Left: seed-level distributions of NodeMSE@32. Right: seed-level distributions of GrowthSlope. The distributions remain stable across seeds, supporting the three-seed protocol used in the main experiments.
  }
  \label{fig:app_seed_stability}
\end{figure}

\subsection{Multi-Hop Error Propagation Under FE Dynamics}
\label{app:exp7_multihop}

We test whether local perturbations propagate across multiple graph hops by injecting perturbations at depths $\{1,2,4,6,8,10\}$ on chain and scale-free graphs across six baselines and three seeds. Figure~\ref{fig:app_additional_exps}(a) shows that propagation remains limited for non-diverged models. Chain graphs remain below NodeMSE@$H=20$ of $10^{-12}$ across injection depths, while non-diverged scale-free runs remain below $10^{-13}$. Severe propagation appears mainly in already-diverged GCN runs, supporting the contractive interpretation of the FE regime.

\subsection{Sparse and Dense Message-Passing Layers}
\label{app:exp13_sparse_dense_mp}

We compare the vanilla GCN world model, MPNN world model, and GPS-style graph transformer on the scale-free topology over horizons $H\in\{1,4,8,16,32\}$ and three seeds. Figure~\ref{fig:app_additional_exps}(b) shows that GCN diverges with horizon, reaching a median NodeMSE@32 of $6.5\times10^9$, whereas MPNN and GPS remain near the low-error floor at $2.2\times10^{-4}$ and $4.0\times10^{-4}$, respectively. Thus, the observed instability is not explained by message-passing density alone, but depends on the interaction between the transition operator and graph topology.

\subsection{Temporal Memory Variants}
\label{app:exp24_memory}

We compare Markov, last-2, and last-4 history variants, together with recurrent, transformer, and retrieval-augmented memory architectures. Figure~\ref{fig:app_additional_exps}(c) shows that additional temporal context reduces long-horizon error, with most of the gain obtained from the last-2 context. Among the extended variants, retrieval-augmented memory achieves the lowest long-horizon error, indicating that temporal context can further improve rollout stability.

\begin{figure*}[!ht]
    \centering
    \begin{minipage}[t]{0.32\textwidth}
        \centering
        \includegraphics[width=\linewidth]{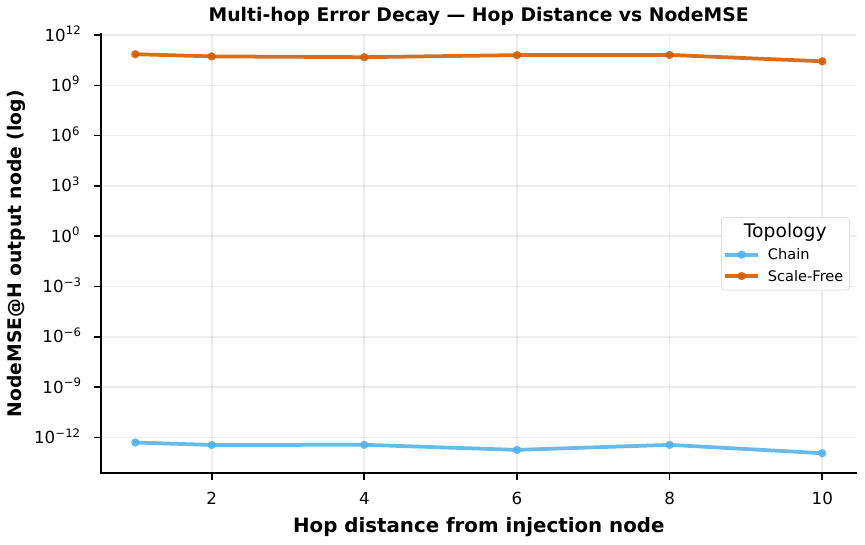}
        \centerline{\small (a) Multi-hop propagation}
    \end{minipage}
    \hfill
    \begin{minipage}[t]{0.32\textwidth}
        \centering
        \includegraphics[width=\linewidth]{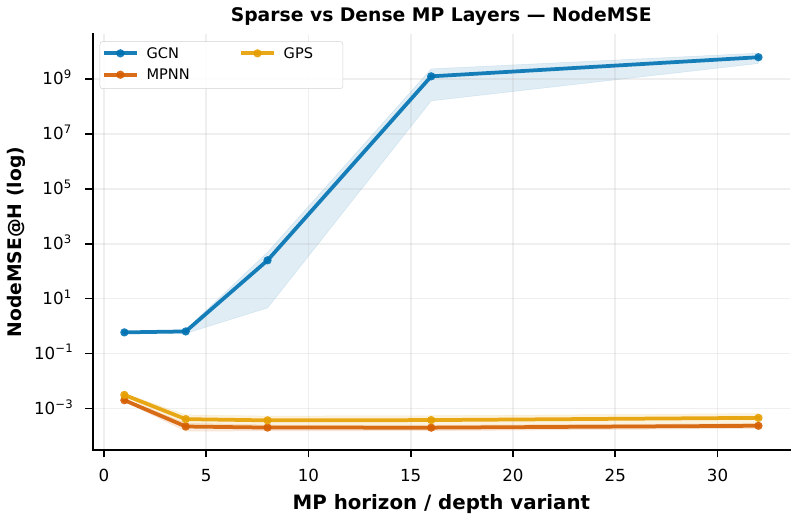}
        \centerline{\small (b) Message passing}
    \end{minipage}
    \hfill
    \begin{minipage}[t]{0.32\textwidth}
        \centering
        \includegraphics[width=\linewidth]{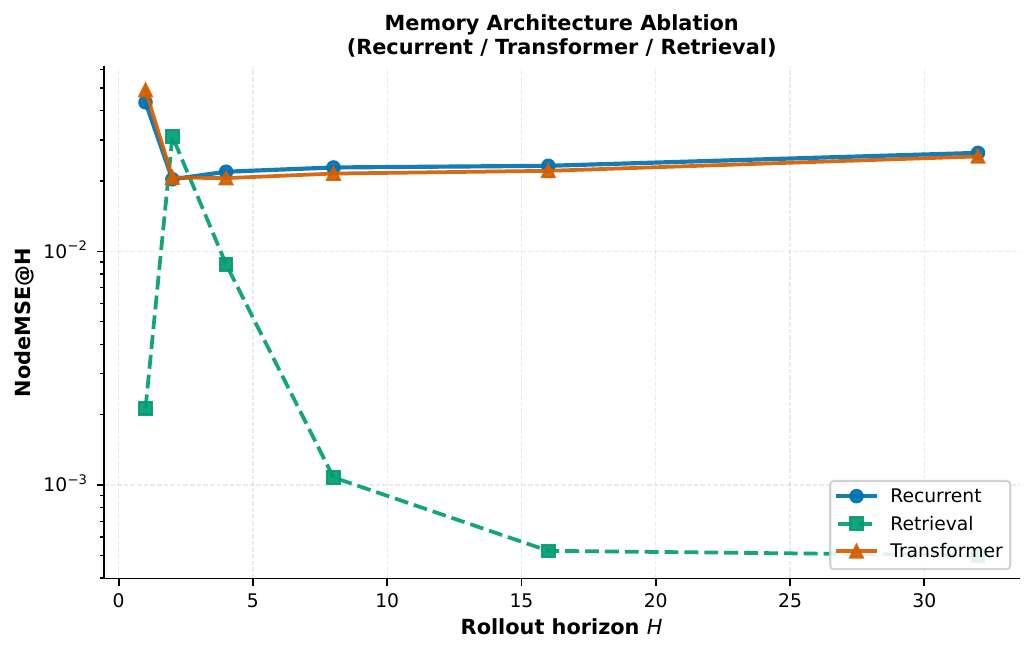}
        \centerline{\small (c) Temporal memory}
    \end{minipage}
    \caption{\textbf{Additional rollout analyses.}
    (a) Multi-hop perturbations remain limited under contractive FE dynamics except in already-diverged runs.
    (b) GCN diverges on scale-free graphs while MPNN and GPS remain stable, showing that message-passing density alone does not explain instability.
    (c) Temporal context improves long-horizon prediction, with retrieval-augmented memory providing the strongest reduction among the evaluated memory variants.}
    \label{fig:app_additional_exps}
\end{figure*}

\section{Proofs and mathematical background}\label{sec:proofs}

\subsection{Proof of Lemma~\ref{lem:t1_fixed_edge}}
\label{app:proof_t1_fixed_edge}

\begin{proof}
Since the graph is fixed in the FE regime, $A_t\equiv A$ for all rollout steps. Let $F_\theta^X(\cdot,A,a)$ and $F^{X\star}(\cdot,A,a)$ denote the learned and true node-transition maps, respectively. By Assumption~\ref{assump:onestep_error},
\begin{equation}
\|F_\theta^X(X,A,a)-F^{X\star}(X,A,a)\|_F\le \epsilon_X
\label{eq:proof_t1_onestep}
\end{equation}
for every $(X,A,a)$ in the compact state space.

Let $e_k^X:=\|\hat X_k-X_k\|_F$. Adding and subtracting $F_\theta^X(X_k,A,a_k)$, applying the triangle inequality, and then using Equation~\eqref{eq:proof_t1_onestep} give
\begin{equation}
\begin{aligned}
e_{k+1}^X
&=\|F_\theta^X(\hat X_k,A,a_k)-F^{X\star}(X_k,A,a_k)\|_F\\
&=\big\|F_\theta^X(\hat X_k,A,a_k)-F_\theta^X(X_k,A,a_k)
   +F_\theta^X(X_k,A,a_k)-F^{X\star}(X_k,A,a_k)\big\|_F\\
&\le \|F_\theta^X(\hat X_k,A,a_k)-F_\theta^X(X_k,A,a_k)\|_F
   +\|F_\theta^X(X_k,A,a_k)-F^{X\star}(X_k,A,a_k)\|_F\\
&\le \|F_\theta^X(\hat X_k,A,a_k)-F_\theta^X(X_k,A,a_k)\|_F+\epsilon_X.
\end{aligned}
\label{eq:proof_t1_error_decomp}
\end{equation}

Under Assumption~\ref{assump:message_passing}, repeated application of the coordinate-wise $L_\sigma$-Lipschitz property of $\sigma$ and submultiplicativity of the spectral and Frobenius norms yields
\begin{equation}
\begin{aligned}
\|F_\theta^X(\hat X_k,A,a_k)-F_\theta^X(X_k,A,a_k)\|_F
&\le L_\sigma\|A\|_2\prod_\ell\|W_\ell\|_2
   \|\hat X_k-X_k\|_F\\
&=L_\sigma\rho(A)\prod_\ell\|W_\ell\|_2\,e_k^X\\
&=:L_X e_k^X,
\end{aligned}
\label{eq:proof_t1_lipschitz}
\end{equation}
where the second equality uses $\|A\|_2=\rho(A)$ for symmetric $A$, and
\begin{equation}
L_X:=L_\sigma\rho(A)\prod_\ell\|W_\ell\|_2.
\label{eq:proof_t1_lx}
\end{equation}

Substituting Equation~\eqref{eq:proof_t1_lipschitz} into Equation~\eqref{eq:proof_t1_error_decomp} gives
\begin{equation}
e_{k+1}^X\le L_Xe_k^X+\epsilon_X.
\label{eq:proof_t1_recursion}
\end{equation}
Repeatedly applying Equation~\eqref{eq:proof_t1_recursion} gives
\begin{equation}
\begin{aligned}
e_k^X
&\le L_Xe_{k-1}^X+\epsilon_X\\
&\le L_X\big(L_Xe_{k-2}^X+\epsilon_X\big)+\epsilon_X\\
&\le L_X^2e_{k-2}^X+\epsilon_X(1+L_X)\\
&\ \vdots\\
&\le L_X^ke_0^X+\epsilon_X\sum_{i=0}^{k-1}L_X^i\\
&=
\begin{cases}
L_X^ke_0^X+\epsilon_X\dfrac{L_X^k-1}{L_X-1}, & L_X\neq 1,\\[6pt]
e_0^X+k\epsilon_X, & L_X=1.
\end{cases}
\end{aligned}
\label{eq:proof_t1_final}
\end{equation}
Equation~\eqref{eq:proof_t1_final} is precisely the claimed FE rollout-error bound, completing the proof.
\end{proof}

\subsection{Proof of Lemma~\ref{lem:t2_joint_operator}}
\label{app:proof_t2_joint_operator}

\begin{proof}
By assumption,
\begin{equation}
\begin{aligned}
e^X_{k+1} &\le L_X e^X_k+L_A e^A_k+\epsilon_X,\\
e^A_{k+1} &\le M_X e^X_k+M_A e^A_k+\epsilon_A.
\end{aligned}
\label{eq:proof_t2_scalar}
\end{equation}
Define
\begin{equation}
u_k:=\begin{pmatrix}e^X_k\\ e^A_k\end{pmatrix},\qquad B:=\begin{pmatrix}L_X & L_A\\ M_X & M_A\end{pmatrix},\qquad \epsilon:=\begin{pmatrix}\epsilon_X\\ \epsilon_A\end{pmatrix}.
\label{eq:proof_t2_defs}
\end{equation}
Then
\begin{equation}
\begin{aligned}
u_{k+1}
&=\begin{pmatrix}e^X_{k+1}\\ e^A_{k+1}\end{pmatrix}\\
&\preceq\begin{pmatrix}L_X e^X_k+L_A e^A_k+\epsilon_X\\ M_X e^X_k+M_A e^A_k+\epsilon_A\end{pmatrix}\\
&=\begin{pmatrix}L_X & L_A\\ M_X & M_A\end{pmatrix}\begin{pmatrix}e^X_k\\ e^A_k\end{pmatrix}+\begin{pmatrix}\epsilon_X\\ \epsilon_A\end{pmatrix}\\
&=Bu_k+\epsilon.
\end{aligned}
\label{eq:proof_t2_vector}
\end{equation}
Equation~\eqref{eq:proof_t2_vector} proves Equation~\eqref{eq:t2_recursion}.

For the eigenvalues of $B$,
\begin{equation}
\begin{aligned}
0
&=\det(\lambda I-B)\\
&=\det\begin{pmatrix}\lambda-L_X & -L_A\\ -M_X & \lambda-M_A\end{pmatrix}\\
&=(\lambda-L_X)(\lambda-M_A)-L_AM_X\\
&=\lambda^2-(L_X+M_A)\lambda+(L_XM_A-L_AM_X).
\end{aligned}
\label{eq:proof_t2_charpoly}
\end{equation}
Hence,
\begin{equation}
\begin{aligned}
\lambda_{\pm}
&=\frac{L_X+M_A\pm\sqrt{(L_X+M_A)^2-4(L_XM_A-L_AM_X)}}{2}\\
&=\frac{L_X+M_A\pm\sqrt{L_X^2+2L_XM_A+M_A^2-4L_XM_A+4L_AM_X}}{2}\\
&=\frac{L_X+M_A\pm\sqrt{L_X^2-2L_XM_A+M_A^2+4L_AM_X}}{2}\\
&=\frac{1}{2}\left[L_X+M_A\pm\sqrt{(L_X-M_A)^2+4L_AM_X}\right].
\end{aligned}
\label{eq:proof_t2_eigenvalues}
\end{equation}
Since $B$ is nonnegative, Perron--Frobenius implies
\begin{equation}
\begin{aligned}
\rho(B)
&=\lambda_+\\
&=\frac{1}{2}\left[L_X+M_A+\sqrt{(L_X-M_A)^2+4L_AM_X}\right].
\end{aligned}
\label{eq:proof_t2_rho}
\end{equation}
Equation~\eqref{eq:proof_t2_rho} proves Equation~\eqref{eq:joint_operator_spectral_radius}.

If $L_AM_X>0$, then
\begin{equation}
\begin{aligned}
(L_X-M_A)^2+4L_AM_X
&>(L_X-M_A)^2,\\
\sqrt{(L_X-M_A)^2+4L_AM_X}
&>|L_X-M_A|.
\end{aligned}
\label{eq:proof_t2_strict_sqrt}
\end{equation}
Using Equations~\eqref{eq:proof_t2_rho} and~\eqref{eq:proof_t2_strict_sqrt},
\begin{equation}
\begin{aligned}
\rho(B)
&=\frac{1}{2}\left[L_X+M_A+\sqrt{(L_X-M_A)^2+4L_AM_X}\right]\\
&>\frac{1}{2}\left[L_X+M_A+|L_X-M_A|\right]\\
&=\max\{L_X,M_A\}.
\end{aligned}
\label{eq:proof_t2_strict}
\end{equation}
This proves Lemma~\ref{lem:t2_joint_operator}.
\end{proof}

\subsection{Proof of Lemma~\ref{lem:t3_geaf_proxy}}
\label{app:proof_t3_geaf_proxy}

\begin{proof}
Since $B$ is nonnegative, the Gershgorin circle theorem gives
\begin{equation}
\begin{aligned}
\rho(B)
&\le \max\{L_X+L_A,M_X+M_A\}.
\end{aligned}
\label{eq:proof_t3_gershgorin}
\end{equation}
In the signal-dominant regime, the first-row error term satisfies
\begin{equation}
\begin{aligned}
L_X+L_A
&\le \GEAF(G)+\GEAF(G)\frac{R_X}{\|A\|_2}\\
&=\GEAF(G)\left(1+\frac{R_X}{\|A\|_2}\right).
\end{aligned}
\label{eq:proof_t3_first_row}
\end{equation}
Under the same condition, $M_X+M_A$ is no larger than the bound in Equation~\eqref{eq:proof_t3_first_row}. Therefore, using Equations~\eqref{eq:proof_t3_gershgorin} and~\eqref{eq:proof_t3_first_row} together with $\|A\|_2\ge R_X$,
\begin{equation}
\begin{aligned}
\rho(B)
&\le \max\{L_X+L_A,M_X+M_A\}\\
&\le \GEAF(G)\left(1+\frac{R_X}{\|A\|_2}\right)\\
&=2\GEAF(G).
\end{aligned}
\label{eq:proof_t3_final_bound}
\end{equation}
Equation~\eqref{eq:proof_t3_final_bound} proves Equation~\eqref{eq:geaf_proxy_bound}.

Finally, when the model weights are fixed across topologies, let $C_W:=\prod_\ell\|W_\ell\|_2$, which is constant across graphs. Then, for any two topologies $G_i$ and $G_j$ with adjacency matrices $A_i$ and $A_j$,
\begin{equation}
\begin{aligned}
\GEAF(G_j)-\GEAF(G_i)
&=\rho(A_j)C_W-\rho(A_i)C_W\\
&=C_W\bigl(\rho(A_j)-\rho(A_i)\bigr).
\end{aligned}
\label{eq:proof_t3_rank_difference}
\end{equation}
Since $C_W\ge 0$, Equation~\eqref{eq:proof_t3_rank_difference} shows that the sign of $\GEAF(G_j)-\GEAF(G_i)$ is determined by the sign of $\rho(A_j)-\rho(A_i)$. Hence, $\GEAF(G)$ is monotone in $\rho(A)$ and induces the same topology ranking, completing the proof.
\end{proof}

\subsection{Proof of Lemma~\ref{lem:t4_regret}}
\label{app:proof_t4_regret}

\begin{proof}
Let $\pi^*$ be the policy selected by planning with the true dynamics and let $\hat\pi_\theta$ be the policy selected by planning with the learned GWM. By the simulation lemma,
\begin{equation}
\begin{aligned}
J^*(\pi^*)-J^*(\hat\pi_\theta)
&\le 2\sum_{k=0}^{H-1}\gamma^k\left|R^\star(G_k,a_k)-R_\theta(\hat G_k,a_k)\right|.
\end{aligned}
\label{eq:proof_t4_sim_start}
\end{equation}
Adding and subtracting $R^\star(\hat G_k,a_k)$ and applying Assumption~\ref{assump:reward},
\begin{equation}
\begin{aligned}
\left|R^\star(G_k,a_k)-R_\theta(\hat G_k,a_k)\right|
&\le \left|R^\star(G_k,a_k)-R^\star(\hat G_k,a_k)\right|+\left|R^\star(\hat G_k,a_k)-R_\theta(\hat G_k,a_k)\right|\\
&\le L_R d_G(G_k,\hat G_k)+\epsilon_R\\
&=L_R c^\top u_k+\epsilon_R\\
&\le L_R\|c\|_2\|u_k\|_2+\epsilon_R,
\end{aligned}
\label{eq:proof_t4_reward_bound}
\end{equation}

From Lemma~\ref{lem:t2_joint_operator},
\begin{equation}
\begin{aligned}
u_k
&\preceq B^k u_0+\sum_{i=0}^{k-1}B^i\epsilon.
\end{aligned}
\label{eq:proof_t4_recursion_unroll}
\end{equation}
For rollouts initialized from the observed graph, $u_0=0$. Let $\tilde\epsilon$ be a uniform scale for the one-step error vector and let $\kappa(B)$ denote the matrix-dependent constant satisfying $\|B^i\|_2\le\kappa(B)\rho(B)^i$. Equation~\eqref{eq:proof_t4_recursion_unroll} then gives
\begin{equation}
\begin{aligned}
\|u_k\|_2
&\le \sum_{i=0}^{k-1}\|B^i\epsilon\|_2\\
&\le \sum_{i=0}^{k-1}\|B^i\|_2\|\epsilon\|_2\\
&\le \kappa(B)\tilde\epsilon\sum_{i=0}^{k-1}\rho(B)^i.
\end{aligned}
\label{eq:proof_t4_uk_bound}
\end{equation}
Substituting Equations~\eqref{eq:proof_t4_reward_bound} and~\eqref{eq:proof_t4_uk_bound} into Equation~\eqref{eq:proof_t4_sim_start},
\begin{equation}
\begin{aligned}
J^*(\pi^*)-J^*(\hat\pi_\theta)
&\le 2\sum_{k=0}^{H-1}\gamma^k\left(L_R\|c\|_2\|u_k\|_2+\epsilon_R\right)\\
&\le 2L_R\|c\|_2\kappa(B)\tilde\epsilon\sum_{k=0}^{H-1}\gamma^k\sum_{i=0}^{k-1}\rho(B)^i+2\epsilon_R\sum_{k=0}^{H-1}\gamma^k\\
&=2L_R\|c\|_2\kappa(B)\tilde\epsilon\sum_{k=0}^{H-1}\gamma^k\sum_{i=0}^{k-1}\rho(B)^i+2\epsilon_R\frac{1-\gamma^H}{1-\gamma}.
\end{aligned}
\label{eq:proof_t4_double_sum}
\end{equation}
Writing $\rho=\rho(B)$, the rollout-error term in Equation~\eqref{eq:proof_t4_double_sum} satisfies
\begin{equation}
\begin{aligned}
\sum_{k=0}^{H-1}\gamma^k\sum_{i=0}^{k-1}\rho^i
&=\sum_{k=0}^{H-1}\gamma^k\frac{\rho^k-1}{\rho-1}\\
&=\frac{1}{\rho-1}\sum_{k=0}^{H-1}\left((\gamma\rho)^k-\gamma^k\right)\\
&=\frac{1}{\rho-1}\left(\sum_{k=0}^{H-1}(\gamma\rho)^k-\sum_{k=0}^{H-1}\gamma^k\right)\\
&=\frac{1}{\rho-1}\left(\frac{1-(\gamma\rho)^H}{1-\gamma\rho}-\frac{1-\gamma^H}{1-\gamma}\right)\\
&=\Phi_H(\gamma,\rho).
\end{aligned}
\label{eq:proof_t4_phi}
\end{equation}
Using Equation~\eqref{eq:proof_t4_phi} with $\rho=\rho(B)$ in Equation~\eqref{eq:proof_t4_double_sum},
\begin{equation}
\begin{aligned}
J^*(\pi^*)-J^*(\hat\pi_\theta)
&\le 2L_R\|c\|_2\kappa(B)\tilde\epsilon\Phi_H(\gamma,\rho(B))+2\epsilon_R\frac{1-\gamma^H}{1-\gamma}.
\end{aligned}
\label{eq:proof_t4_final}
\end{equation}
Finally, when $\gamma\rho(B)>1$, Equation~\eqref{eq:proof_t4_phi} contains the exponentially increasing factor $(\gamma\rho(B))^H$, so the rollout-error contribution grows faster than linearly with the planning horizon $H$. This proves Lemma~\ref{lem:t4_regret}.
\end{proof}

\subsection{Proof of Corollary~\ref{cor:slope_vs_level}}
\label{app:proof_slope_vs_level}

\begin{proof}
From Lemma~\ref{lem:t1_fixed_edge}, the FE rollout error satisfies
\begin{equation}
\begin{aligned}
e^X_k
&\le L_X^k e^X_0+\epsilon_X\frac{L_X^k-1}{L_X-1}\\
&=L_X^k\left(e^X_0+\frac{\epsilon_X}{L_X-1}\right)-\frac{\epsilon_X}{L_X-1}.
\end{aligned}
\label{eq:proof_slope_t1}
\end{equation}
When the rollout is not dominated by the constant error floor, define $C:=e^X_0+\epsilon_X/(L_X-1)$. The leading $k$-dependence in Equation~\eqref{eq:proof_slope_t1} is then
\begin{equation}
\begin{aligned}
e^X_k
&\approx C L_X^k\\
\log e^X_k
&\approx \log\left(C L_X^k\right)\\
&=\log C+\log L_X^k\\
&=\log C+k\log L_X.
\end{aligned}
\label{eq:proof_slope_log}
\end{equation}
Taking the slope with respect to the rollout step $k$ in Equation~\eqref{eq:proof_slope_log} gives
\begin{equation}
\begin{aligned}
\frac{\partial\log e^X_k}{\partial k}
&\approx \frac{\partial}{\partial k}\left(\log C+k\log L_X\right)\\
&=\log L_X,
\end{aligned}
\label{eq:proof_slope_final}
\end{equation}
and hence $\partial\log e^X_k/\partial k\to\log L_X$ in the amplification-dominated regime. This proves Corollary~\ref{cor:slope_vs_level}.
\end{proof}

\section{Auxiliary latent world-model bounds}
\label{sec:problem_formulation2}

\subsection{Latent world-model setup}
This appendix records the scalar latent-world-model bounds that underlie the graph-valued results in the main text. They are not a replacement for Lemmas~\ref{lem:t1_fixed_edge}--\ref{lem:t4_regret}; rather, they show how the familiar compounding-error argument for latent dynamics~\cite{hafner2018dreamer,chua2018pets,janner2019trust} becomes the graph-specific operator analysis used in Section~\ref{sec:theory}. We consider a sequential decision-making setting in which an agent interacts with a partially observed environment over discrete time steps. At each time step $t$, the agent receives an observation $o_t \in \mathcal{O}$, selects an action $a_t \in \mathcal{A}$, and receives a scalar reward $r_t \in \mathbb{R}$. Since the true environment state is generally not directly observable, the agent must infer a compact latent state from its interaction history. Let
\begin{equation}
    h_t = (o_1, a_1, o_2, a_2, \ldots, o_t)
\end{equation}
denote the observable history up to time $t$. A latent world model learns an internal state representation
\begin{equation}
    z_t = e_\theta(h_t),
\end{equation}
where $z_t \in \mathcal{Z}$ is a latent state and $e_\theta$ is an encoder parameterized by $\theta$.

The purpose of the latent world model is not only to encode the current observation, but also to predict how the latent state evolves under actions. We define a latent world model as
\begin{equation}
    \begin{aligned}
    \mathcal{W}_\theta =
    (e_\theta, 
    p_\theta(z_{t+1} \mid z_t, a_t),p_\theta(o_t \mid z_t), 
    p_\theta(r_t \mid z_t, a_t)),
    \end{aligned}
\end{equation}
where $e_\theta$ maps observation histories to latent states, $p_\theta(z_{t+1} \mid z_t, a_t)$ is the latent transition model, $p_\theta(o_t \mid z_t)$ is the observation reconstruction or prediction model, and $p_\theta(r_t \mid z_t, a_t)$ is the reward prediction model.

Given a trajectory
\begin{equation}
    \tau = (o_1, a_1, r_1, o_2, a_2, r_2, \ldots, o_T),
\end{equation}
the latent world model induces the following factorization:
\begin{equation}
\begin{aligned}
    p_\theta(o_{1:T}, r_{1:T}, z_{1:T} \mid a_{1:T})
    =
    p_\theta(z_1 \mid o_1)\prod_{t=1}^{T}  p_\theta(o_t \mid z_t)
    p_\theta(r_t \mid z_t, a_t)
    p_\theta(z_{t+1} \mid z_t, a_t).
    \end{aligned}
\end{equation}
This factorization shows that the model compresses the observable history into latent states while preserving the information needed to predict future observations and rewards.

The learning objective is to fit the latent world model to trajectories collected from the environment. A common training objective is the negative log-likelihood
\begin{equation}
\begin{aligned}
    &\mathcal{L}_{\mathrm{WM}}(\theta)
    =
    - \sum_{t=1}^{T}
    \Big[
    \log p_\theta(o_t \mid z_t)+
    \log p_\theta(r_t \mid z_t, a_t)
    +
    \log p_\theta(z_{t+1} \mid z_t, a_t)
    \Big],
    \end{aligned}
\end{equation}
where $z_t = e_\theta(h_t)$. In variational latent world models, the latent transition term can be replaced by an evidence lower bound that includes an inference model $q_\theta(z_t \mid h_t)$ and a prior transition model $p_\theta(z_t \mid z_{t-1}, a_{t-1})$.

After training, the latent world model can be used for multi-step imagination or planning. Starting from an inferred latent state $z_t$, the model recursively predicts future latent states under a candidate action sequence $(a_t, a_{t+1}, \ldots, a_{t+H-1})$:
\begin{equation}
    \hat{z}_{t+k+1}
    \sim
    p_\theta(\hat{z}_{t+k+1} \mid \hat{z}_{t+k}, a_{t+k}),
\end{equation}
    $k = 0, 1, \ldots, H-1$, with $\hat{z}_t = z_t$. The predicted return of this imagined trajectory is
\begin{equation}
    \hat{J}_\theta(a_{t:t+H-1} \mid z_t)
    =
    \mathbb{E}_{\mathcal{W}_\theta}
    \left[
    \sum_{k=0}^{H-1}
    \gamma^k
    \hat{r}_{t+k}
    \right],
\end{equation}
where
\begin{equation}
    \hat{r}_{t+k}
    \sim
    p_\theta(\hat{r}_{t+k} \mid \hat{z}_{t+k}, a_{t+k}).
\end{equation}
The planning problem can then be written as
\begin{equation}
    a_{t:t+H-1}^{\star}
    =
    \arg\max_{a_{t:t+H-1}}
    \hat{J}_\theta(a_{t:t+H-1} \mid z_t).
\end{equation}

Thus, a latent world model is an action-conditioned predictive representation of the environment. It learns a latent state space in which the agent can simulate future trajectories, estimate their rewards, and select actions according to predicted long-term outcomes.

\subsection{Error Propagation in Latent World Models}
\label{sec:error_propagation}

We first study how a one-step latent transition error propagates through multi-step imagination. We consider the deterministic mean dynamics induced by the latent world model. Let
\begin{equation}
    \hat{z}_{t+1} = F_\theta(\hat{z}_t, a_t)
\end{equation}
denote the learned latent transition, and let
\begin{equation}
    z_{t+1} = F^\star(z_t, a_t)
\end{equation}
denote the latent dynamics induced by the true environment under the same latent representation. Here, $F^\star$ is not required to be a physical-state transition; it is the true transition map in the latent space.

\begin{assumption}[Uniform one-step latent transition error]
\label{ass:transition_error}
There exists a constant $\epsilon_F \ge 0$ such that
\begin{equation}
    \sup_{z \in \mathcal{Z}, a \in \mathcal{A}}
    \left\|F_\theta(z,a) - F^\star(z,a)\right\|
    \le \epsilon_F .
\end{equation}
\end{assumption}

\begin{assumption}[Lipschitz continuity of the learned latent transition]
\label{ass:transition_lipschitz}
There exists a constant $L_F \ge 0$ such that, for any $z,z' \in \mathcal{Z}$ and $a \in \mathcal{A}$,
\begin{equation}
    \left\|F_\theta(z,a) - F_\theta(z',a)\right\|
    \le L_F \left\|z-z'\right\|.
\end{equation}
\end{assumption}

Given an action sequence $a_{t:t+H-1}$, define the true latent rollout and the imagined latent rollout as
\begin{equation}
\begin{aligned}
    z_{t+k+1} &= F^\star(z_{t+k}, a_{t+k}),
    \\
    \hat{z}_{t+k+1} &= F_\theta(\hat{z}_{t+k}, a_{t+k}),
    \end{aligned}
\end{equation}
for $k=0,1,\ldots,H-1$. Define the $k$-step rollout error by
\begin{equation}
    e_k = \left\|\hat{z}_{t+k} - z_{t+k}\right\|.
\end{equation}

\begin{theorem}[Multi-step latent transition error propagation]
\label{thm:latent_transition_error}
Suppose Assumptions~\ref{ass:transition_error} and~\ref{ass:transition_lipschitz} hold. Then, for any fixed action sequence $a_{t:t+H-1}$ and any $k\in\{0,1,\ldots,H\}$,
\begin{equation}
e_k\le L_F^k e_0+\epsilon_F\sum_{i=0}^{k-1}L_F^i.
\label{eq:latent_transition_error_general}
\end{equation}
In particular, if $e_0=0$,
\begin{equation}
e_k=\begin{cases}O(\epsilon_F), & 0\le L_F<1,\\ O(k\epsilon_F), & L_F=1,\\ O(\epsilon_F L_F^k), & L_F>1.\end{cases}
\label{eq:latent_transition_error_regimes}
\end{equation}
\end{theorem}

\begin{proof}
For each $k=0,1,\ldots,H-1$, adding and subtracting $F_\theta(z_{t+k},a_{t+k})$ and applying Assumptions~\ref{ass:transition_lipschitz} and~\ref{ass:transition_error} give
\begin{equation}
\begin{aligned}
e_{k+1}
&=\left\|F_\theta(\hat z_{t+k},a_{t+k})-F^\star(z_{t+k},a_{t+k})\right\|\\
&\le \left\|F_\theta(\hat z_{t+k},a_{t+k})-F_\theta(z_{t+k},a_{t+k})\right\|+\left\|F_\theta(z_{t+k},a_{t+k})-F^\star(z_{t+k},a_{t+k})\right\|\\
&\le L_F\left\|\hat z_{t+k}-z_{t+k}\right\|+\epsilon_F\\
&=L_F e_k+\epsilon_F.
\end{aligned}
\label{eq:latent_error_recursion}
\end{equation}
Unrolling Equation~\eqref{eq:latent_error_recursion} yields
\begin{equation}
\begin{aligned}
e_k
&\le L_F e_{k-1}+\epsilon_F\\
&\le L_F\left(L_F e_{k-2}+\epsilon_F\right)+\epsilon_F\\
&=L_F^2e_{k-2}+\epsilon_F(1+L_F)\\
&\le \cdots\\
&\le L_F^k e_0+\epsilon_F\sum_{i=0}^{k-1}L_F^i.
\end{aligned}
\label{eq:latent_error_unrolled}
\end{equation}
For $L_F\neq 1$, Equation~\eqref{eq:latent_error_unrolled} gives
\begin{equation}
\begin{aligned}
e_k
&\le L_F^k e_0+\epsilon_F\sum_{i=0}^{k-1}L_F^i\\
&=L_F^k e_0+\epsilon_F\frac{L_F^k-1}{L_F-1}.
\end{aligned}
\label{eq:latent_error_closed_form}
\end{equation}
For $L_F=1$, Equation~\eqref{eq:latent_error_recursion} gives
\begin{equation}
\begin{aligned}
e_k
&\le e_{k-1}+\epsilon_F\\
&\le e_{k-2}+2\epsilon_F\\
&\le \cdots\\
&\le e_0+k\epsilon_F.
\end{aligned}
\label{eq:latent_error_unit_proof}
\end{equation}
Equations~\eqref{eq:latent_error_unrolled}--\eqref{eq:latent_error_unit_proof} establish the claimed bounds.
\end{proof}

Theorem~\ref{thm:latent_transition_error} shows that even a uniformly bounded one-step transition error can increase across a multi-step imagined trajectory. When $L_F<1$, the transition is contractive and the accumulated error remains controlled. When $L_F=1$, the error grows at most linearly in the planning horizon. When $L_F>1$, the error bound grows geometrically, which gives a formal explanation of compounding error in latent world models.

\subsection{Reward Prediction and Return Error Bound}
\label{sec:reward_error_bound}

We next show that latent transition error induces reward prediction error and therefore biases the planning objective. Let the learned reward model be
\begin{equation}
    \hat{r}_{t+k} = R_\theta(\hat{z}_{t+k}, a_{t+k}),
\end{equation}
and let the true latent reward function be
\begin{equation}
    r_{t+k} = R^\star(z_{t+k}, a_{t+k}).
\end{equation}
For a fixed horizon $H$, define the learned imagined return and the true return under the same action sequence as
\begin{equation}
    \hat{J}_\theta(a_{t:t+H-1}\mid \hat{z}_t)
    =
    \sum_{k=0}^{H-1}
    \gamma^k
    R_\theta(\hat{z}_{t+k}, a_{t+k})
\end{equation}
and
\begin{equation}
    J^\star(a_{t:t+H-1}\mid z_t)
    =
    \sum_{k=0}^{H-1}
    \gamma^k
    R^\star(z_{t+k}, a_{t+k}).
\end{equation}

\begin{assumption}[Uniform one-step reward model error]
\label{ass:reward_error}
There exists a constant $\epsilon_R \ge 0$ such that
\begin{equation}
    \sup_{z \in \mathcal{Z}, a \in \mathcal{A}}
    \left|
    R_\theta(z,a) - R^\star(z,a)
    \right|
    \le \epsilon_R .
\end{equation}
\end{assumption}

The previous theorem bounds the deviation of the imagined latent trajectory from the true latent trajectory. We now show that this latent rollout error directly induces reward prediction error, because the reward model is evaluated on the imagined latent states rather than on the true latent states.

\begin{assumption}[Lipschitz continuity of the learned reward model]
\label{ass:reward_lipschitz}
There exists a constant $L_R \ge 0$ such that, for any $z,z' \in \mathcal{Z}$ and $a \in \mathcal{A}$,
\begin{equation}
    \left|
    R_\theta(z,a) - R_\theta(z',a)
    \right|
    \le
    L_R \left\|z-z'\right\|.
\end{equation}
\end{assumption}

\begin{proposition}[Reward prediction error bound]
\label{prop:reward_prediction_error}
Suppose Assumptions~\ref{ass:transition_error},~\ref{ass:transition_lipschitz},~\ref{ass:reward_error}, and~\ref{ass:reward_lipschitz} hold. For any fixed action sequence $a_{t:t+H-1}$ and any $k=0,1,\ldots,H-1$, define
\begin{equation}
    \Delta^R_k
    =
    \left|
    R_\theta(\hat{z}_{t+k},a_{t+k})
    -
    R^\star(z_{t+k},a_{t+k})
    \right| .
\end{equation}
Then
\begin{equation}
\begin{aligned}
    \Delta^R_k
    \le
    L_R
    \left(
    L_F^k e_0
    +
    \epsilon_F
    \sum_{i=0}^{k-1}L_F^i
    \right)
    +
    \epsilon_R .
    \end{aligned}
\end{equation}
In particular, if $e_0=0$, then
\begin{equation}
    \Delta^R_k
    \le
    L_R\epsilon_F
    \sum_{i=0}^{k-1}L_F^i
    +
    \epsilon_R .
\end{equation}
\end{proposition}

Proposition~\ref{prop:reward_prediction_error} separates the reward prediction error into two sources. The first term comes from latent state mismatch and scales with the transition error accumulated through the rollout. The second term, $\epsilon_R$, is the intrinsic one-step approximation error of the reward model. Therefore, even when the reward model is uniformly accurate at one step, the reward estimate can still degrade over a long imagined trajectory because the latent inputs themselves become increasingly biased.

\begin{proof}
For any $k=0,1,\ldots,H-1$, adding and subtracting $R_\theta(z_{t+k},a_{t+k})$ and applying Assumptions~\ref{ass:reward_lipschitz} and~\ref{ass:reward_error} give
\begin{equation}
\begin{aligned}
\left|R_\theta(\hat z_{t+k},a_{t+k})-R^\star(z_{t+k},a_{t+k})\right|
&\le \left|R_\theta(\hat z_{t+k},a_{t+k})- R_\theta(z_{t+k},a_{t+k})\right|\\ &+\left|R_\theta(z_{t+k},a_{t+k})-R^\star(z_{t+k},a_{t+k})\right|\\
&\le L_R\left\|\hat z_{t+k}-z_{t+k}\right\|+\epsilon_R\\
&=L_R e_k+\epsilon_R.
\end{aligned}
\label{eq:reward_error_step}
\end{equation}
By Theorem~\ref{thm:latent_transition_error},
\begin{equation}
\begin{aligned}
\left|R_\theta(\hat z_{t+k},a_{t+k})-R^\star(z_{t+k},a_{t+k})\right|
&\le L_R e_k+\epsilon_R\\
&\le L_R\left(L_F^k e_0+\epsilon_F\sum_{i=0}^{k-1}L_F^i\right)+\epsilon_R\\
&=L_RL_F^k e_0+L_R\epsilon_F\sum_{i=0}^{k-1}L_F^i+\epsilon_R.
\end{aligned}
\label{eq:reward_error_final}
\end{equation}
Equation~\eqref{eq:reward_error_final} establishes the claimed reward-error bound.
\end{proof}

\begin{proposition}[Multi-step return error bound]
\label{prop:return_error}
Suppose Assumptions~\ref{ass:transition_error},~\ref{ass:transition_lipschitz},~\ref{ass:reward_error}, and~\ref{ass:reward_lipschitz} hold. Then, for any fixed action sequence $a_{t:t+H-1}$,
\begin{equation}
\left|\hat{J}_\theta(a_{t:t+H-1}\mid \hat{z}_t)-J^\star(a_{t:t+H-1}\mid z_t)\right|\le \sum_{k=0}^{H-1}\gamma^k\left[L_R\left(L_F^k e_0+\epsilon_F\sum_{i=0}^{k-1}L_F^i\right)+\epsilon_R\right].
\label{eq:return_error_general}
\end{equation}
In particular, when $e_0=0$, $L_F=1$, and $\gamma=1$,
\begin{equation}
\begin{aligned}
\left|\hat{J}_\theta(a_{t:t+H-1}\mid z_t)-J^\star(a_{t:t+H-1}\mid z_t)\right|
&\le \frac{L_R\epsilon_F H(H-1)}{2}+\epsilon_R H\\
&=O\!\left(L_R\epsilon_F H^2+\epsilon_R H\right).
\end{aligned}
\label{eq:return_error_special}
\end{equation}
\end{proposition}

Proposition~\ref{prop:return_error} shows that transition error affects planning through the cumulative return objective. When $e_0=0$, $L_F=1$, and $\gamma=1$, the return error scales as $O(L_R\epsilon_F H^2+\epsilon_R H)$. Thus, a one-step latent transition error of order $\epsilon_F$ can lead to a quadratic dependence on the planning horizon, while the reward approximation error accumulates linearly. This explains why a latent world model with small one-step prediction error may still produce biased long-horizon planning estimates.

\begin{proof}
By definition, the triangle inequality and Proposition~\ref{prop:reward_prediction_error} give
\begin{equation}
\begin{aligned}
\left|\hat{J}_\theta(a_{t:t+H-1}\mid \hat z_t)-J^\star(a_{t:t+H-1}\mid z_t)\right|
&=\left|\sum_{k=0}^{H-1}\gamma^k\left[R_\theta(\hat z_{t+k},a_{t+k})-R^\star(z_{t+k},a_{t+k})\right]\right|\\
&\le \sum_{k=0}^{H-1}\gamma^k\left|R_\theta(\hat z_{t+k},a_{t+k})-R^\star(z_{t+k},a_{t+k})\right|\\
&\le \sum_{k=0}^{H-1}\gamma^k\left[L_R\left(L_F^k e_0+\epsilon_F\sum_{i=0}^{k-1}L_F^i\right)+\epsilon_R\right].
\end{aligned}
\label{eq:proof_return_general}
\end{equation}
For $e_0=0$, Equation~\eqref{eq:proof_return_general} becomes
\begin{equation}
\begin{aligned}
\left|\hat{J}_\theta(a_{t:t+H-1}\mid z_t)-J^\star(a_{t:t+H-1}\mid z_t)\right|
&\le \sum_{k=0}^{H-1}\gamma^k\left(L_R\epsilon_F\sum_{i=0}^{k-1}L_F^i+\epsilon_R\right).
\end{aligned}
\label{eq:proof_return_exact}
\end{equation}
When $L_F=1$ and $\gamma=1$, Equation~\eqref{eq:proof_return_exact} further gives
\begin{equation}
\begin{aligned}
\left|\hat{J}_\theta(a_{t:t+H-1}\mid z_t)-J^\star(a_{t:t+H-1}\mid z_t)\right|
&\le \sum_{k=0}^{H-1}\left(L_R\epsilon_F k+\epsilon_R\right)\\
&=L_R\epsilon_F\sum_{k=0}^{H-1}k+\epsilon_R H\\
&=\frac{L_R\epsilon_F H(H-1)}{2}+\epsilon_R H\\
&=O\!\left(L_R\epsilon_F H^2+\epsilon_R H\right).
\end{aligned}
\label{eq:proof_return_final}
\end{equation}
Equation~\eqref{eq:proof_return_final} proves the stated return-error bound and asymptotic rate.
\end{proof}

\end{document}